\documentclass[10pt,twocolumn,letterpaper]{article}

\usepackage[pagenumbers]{cvpr} 

\usepackage{algorithm}
\usepackage{amsthm}
\usepackage{multirow}

\newtheorem{lemma}{Lemma}
\newtheorem{proposition}{Proposition}
\newtheorem{definition}{Definition}

\definecolor{cvprblue}{rgb}{0.21,0.49,0.74}
\usepackage[pagebackref,breaklinks,colorlinks,allcolors=cvprblue]{hyperref}

\def\paperID{14682} 
\def\confName{CVPR}
\def\confYear{2026}

\title{H-Sets: Hessian-Guided Discovery of \\ Set-Level Feature Interactions in Image Classifiers}

\author{Ayushi Mehrotra \thanks{Equal contribution} \\ 
California Institute of Technology \\ 
{\tt\small amehrotra@caltech.edu}\\
\and
Dipkamal Bhusal \footnotemark[1]\\
Rochester Institute of Technology\\
{\tt\small db1702@rit.edu}
\and 
Michael Clifford\\
Toyota InfoTech Labs \\  
{\tt\small michael.clifford@toyota.com}
\and 
Nidhi Rastogi\\
Rochester Institute of Technology\\
{\tt\small nxrvse@rit.edu}
}

\begin{document}
\maketitle
\begin{abstract}
Feature attribution methods explain the predictions of deep neural networks by assigning importance scores to individual input features. However, most existing methods focus solely on marginal effects, overlooking \textit{feature interactions}, where groups of features jointly influence model output. Such interactions are especially important in image classification tasks, where semantic meaning often arises from pixel interdependencies rather than isolated features. Existing interaction-based methods for images are either coarse (e.g., superpixel-only) or, fail to satisfy core interpretability axioms. In this work, we introduce \textbf{H-Sets}, a novel two-stage framework for discovering and attributing higher-order feature interactions in image classifiers. First, we detect locally interacting pairs via input Hessians and recursively merge them into semantically coherent sets; segmentation from Segment Anything (SAM) is used as a spatial grouping prior but can be replaced by other segmentations. Second, we \emph{attribute} each set with \textbf{IDG-Vis}, a set-level extension of Integrated Directional Gradients that integrates directional gradients along pixel-space paths and aggregates them with Harsanyi dividends. While Hessians introduce additional compute at the detection stage, this targeted cost consistently yields saliency maps that are sparser and more faithful. Evaluations across VGG, ResNet, DenseNet and MobileNet models on ImageNet and CUB datasets show that H-Sets generate more interpretable and faithful saliency maps compared to existing methods. Our code is available at \url{https://github.com/ayushimehrotra/H-Sets}.
\end{abstract}
    
\section{Introduction}
\label{sec:intro}
Understanding how deep neural networks arrive at their predictions is essential for building trust, ensuring accountability, and debugging failure modes, especially in high-stakes domains like healthcare and autonomous driving. Feature attribution methods~\cite{vanilla, ig, smilkov2017smoothgrad, gig, gradcam, ribeiro2016should, shrikumar2017learning, petsiuk2018rise, fong2017interpretable, fong2019understanding, bhusal2023sok} have emerged as a key tool for model explainability, assigning importance scores to individual input features to explain a model's decision for a specific input. However, most of these methods focus solely on marginal effects, the contribution of each feature in isolation, while overlooking feature interactions, where groups of features jointly influence the model’s output in ways not captured by their individual effects.

Feature interactions are especially prevalent in image classification, where the semantics of an object emerge from pixel interdependencies. While a few prior works have attempted to capture such interactions~\cite{arch, caso, ghorbani2019towards, taylor, faith, moxi}, these approaches are either computationally expensive, restricted to coarse superpixels or, segments, intractable for high-dimensional images, or fail to satisfy key interpretability axioms~\cite{ig}.

\begin{figure}
    \centering
    \includegraphics[width=0.95\linewidth]{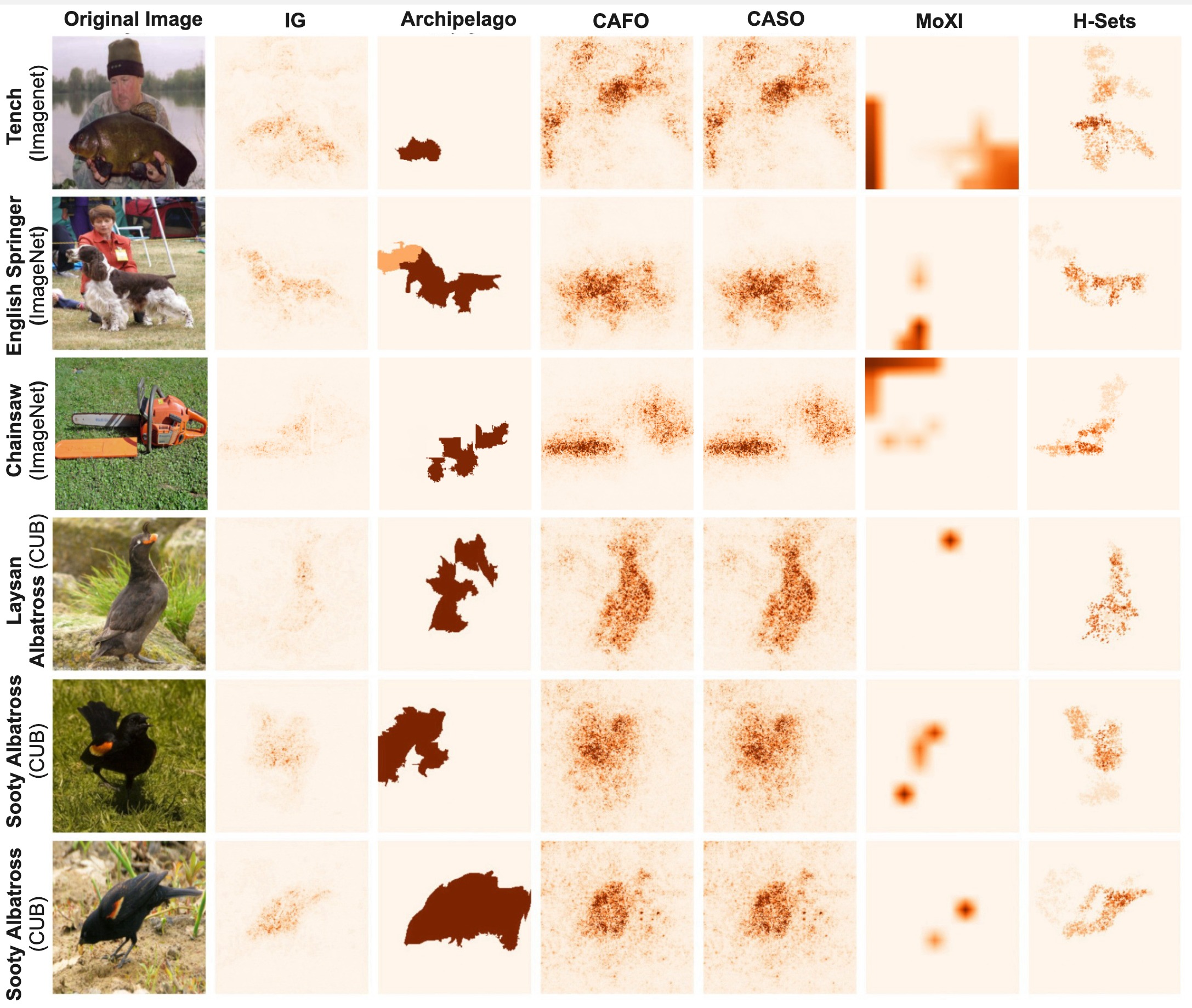}
    \caption{Comparison between Integrated Gradients (IG) \cite{ig}, Archipelago \cite{arch}, Context-aware First Order explanations (CAFO) \cite{caso}, Context-aware Second Order explanations (CASO) \cite{caso}, MOXI~\cite{moxi} and our proposed method, \textbf{H-Sets} on ImageNet and CUB samples with DenseNet121~\cite{resnet}. }
    \label{fig:qual-analy}
\end{figure}

In this work, we propose \textbf{H-Sets}, a principled framework for discovering and attributing higher-order feature interactions in image classifiers. \textbf{H-Sets} operates in two stages. \emph{(1) Interaction Detection:} We detect pairs of features with strong second-order dependencies using the Hessian matrix, then recursively merge them into larger groups while enforcing spatial coherence via segmentation masks from the Segment Anything Model (SAM)~\cite{sam}. Importantly, SAM is used only as a spatial grouping prior and can be replaced by other domain-specific segmentation technique (e.g., QuickShift\cite{zhang2020semantic} or, MedSAM\cite{ma2024segment}). While computing Hessians incurs additional cost than first-order gradients, we restrict their use only to this detection step, which significantly improves the semantic quality of discovered regions and produces far more interpretable saliency maps. \emph{(2) Interaction Attribution:} Once interaction sets are discovered, we assign them scalar importance values using \textbf{IDG-Vis} (Integrated Directional Gradients for Vision), our set-level extension of IDG~\cite{idg}. Unlike IDG, which attributes individual tokens in text, \textbf{IDG-Vis} attributes entire feature sets by integrating directional gradients along paths and combining them with Harsanyi dividends from cooperative game theory.

These design choices give \textbf{H-Sets} several key advantages. The use of Hessians enables it to isolate non-additive feature groups that purely gradient-based methods miss; the use of spatial priors produces coherent regions; and the game-theoretic attribution formulation lets us satisfy a comprehensive set of interpretability axioms. Together, these enable \textbf{H-Sets} to produce saliency maps that are sparser, more faithful, and more comprehensible than existing methods. As illustrated in \cref{fig:qual-analy}, \textbf{H-Sets} consistently identifies concise, and semantically meaningful regions that better reflect the model’s decision process. Quantitative evaluations in ~\cref{section:results} corroborate this observation, showing that these saliency maps achieve higher faithfulness.

Our main contributions are as follows:
\begin{enumerate}
    \item We propose \textbf{H-Sets}, a method that discovers semantically meaningful feature interaction sets by using the Hessian matrix to detect pairwise interactions and recursively expand them into higher-order groups.

\item We introduce \textbf{IDG-Vis}, an adaptation of Integrated Directional Gradients (IDG)~\cite{idg} for vision models, which attributes scalar importance scores to interaction sets using directional gradients and Harsanyi dividends. This approach satisfies key axioms of interpretability from both game theory and attribution literature.

\item  We demonstrate the effectiveness of \textbf{H-Sets} on VGG~\cite{vgg}, ResNet~\cite{resnet}, DenseNet~\cite{densenet} and MobileNet~\cite{mobilenet} models for ImageNet \cite{deng2009imagenet} and CUB~\cite{wah2011caltech} validation sets, achieving superior performance on faithfulness and sparsity metrics compared to existing interaction-based attribution baselines.
\end{enumerate}

\section{Related Work}\label{sec:relatedwork}

Many feature attribution methods explain model predictions by assigning importance scores to input features either via gradients or perturbations~\cite{vanilla, ig, smilkov2017smoothgrad, gig, gradcam, ribeiro2016should,shrikumar2017learning, petsiuk2018rise,fong2017interpretable,fong2019understanding}. However, these methods ignore feature interactions, and capture \emph{marginal} effects of the individual features.

Recent approaches have aimed to capture joint efffects between features. Game-theoretic approaches such as Faith-Shap~\cite{faith} explicitly model interactions building on the Shapley-Taylor interaction index~\cite{taylor} but are impractical at image scale due to subset enumeration. Integrated Hessians detects/attributes pairwise interactions along paths but remains costly~\cite{ih}. Archipelago detects islands over superpixels and attributes group effects~\cite{arch}, while Contextual Decomposition and context-aware variants optimize group-wise perturbations~\cite{cd,caso} without explicitly isolating non-additivity. Recent work \emph{MoXI}~\cite{moxi} proposes efficient self-context Shapley-interaction variants using patch insertion/deletion, reducing complexity from exponential to quadratic. However, its patch-level operation loses granularity over pixel interdependencies. Beyond classification, VX-CODE extends collective contributions to object detectors via patch-wise Shapley interactions~\cite{yamauchi2024explaining}. Outside vision, sparse Möbius-transform recovery~\cite{kang2024learning} and PredDiff-based interaction measures~\cite{blucher2022preddiff} offer principled set scores under low-degree or probabilistic assumptions. Concept based explanation methods explain model predictions in terms of high-level concepts, capturing crops or patches of input images~\cite{ghorbani2019towards, fel2023craft, bhusal2025face}.

Compared to these works, \textbf{H-Sets} detects \emph{local} statistical interactions with input Hessians, then attributes \emph{sets} via IDG-Vis—our set-level extension of IDG~\cite{idg} that integrates directional gradients and aggregates with Harsanyi dividends. Unlike Archipelago~\cite{arch}, detection is curvature-driven (not mask search); unlike Integrated Hessians~\cite{ih}, Hessians are used only for \emph{detection}, while attribution is gradient-path based; unlike MoXI~\cite{moxi}, we operate at pixel resolution (guided by a replaceable spatial prior such as SAM~\cite{sam}). These design choices yield saliency maps that are sparse, and faithful, as demonstrated in \cref{section:results}.

\section{Methodology}

Consider a multi-class image classifier $f$ that maps $\mathbf{x}\in\mathbb{R}^d$ to class logits $\mathbf{z}(\mathbf{x})\in\mathbb{R}^C$ with $f_c(\mathbf{x})$ the logit for class $c$. A standard attribution method explaining $f_c(\mathbf{x})$ assigns an importance score to each feature $x_i$, yielding $\phi(\mathbf{x})\in\mathbb{R}^d$. Our goal is to move beyond marginal (per-feature) effects and explain predictions through \emph{interaction sets}: groups of features whose joint influence on $f_c(\mathbf{x})$ is non-additive and semantically coherent. We denote an interaction set by $\mathcal{I}$ and the collection of such sets by $\mathcal{S}$. 

\subsection{Interaction Set Detection}
Let $\mathcal{I} = \{x_1, ..., x_j\}$ be a subset of input features with size $\mathcal{|I|}$. We denote $\mathbf{x}_\mathcal{I}$ as a vector that includes $x_i$ if $x_i \in \mathcal{I}$ with $0$ in all other indices. The first step to incorporating feature interaction into saliency maps is detecting \textit{semantically meaningful} feature interaction sets. We adopt the notion of statistical non-additivity~\cite{arch}: 
\begin{definition}[Statistical non-additive interaction]\label{def:nonadd}
A function $f$ contains a non-additive interaction over $\mathcal{I}$ if it cannot be decomposed as a sum of $|\mathcal{I}|$ subfunctions each excluding one variable:
$
f(\mathbf{x}) \neq \sum_{i\in\mathcal{I}} f_i(\mathbf{x}_{/\{i\}}).
$
\end{definition}

This definition formalizes the notion that features in $\mathcal{I}$ collectively affect the prediction in a way that individual contributions cannot explain. We refer to effects from single features as \textit{main effects}, and those arising from combinations as \textit{interaction effects}.

Our approach explores higher-order interactions, where $|\mathcal{I}|>2$, as they reveal the compounded effect of multiple features acting together. Notably, if a higher-order interaction exists, then \textit{all of its subsets are also interactions} \cite{higherorder}. We therefore first detect strong pairwise interaction and then expand them into larger sets.

\paragraph{Pairwise interactions via the Hessian.}
Gradient-based methods using mixed partial derivatives are a common approach to detect feature interactions in predictive models~\cite{friedman2008predictive}. If a function $f$ exhibits \textit{statistical interaction} between features ${x}_i \in \mathcal{I}$, then $\mathbb{E}_x\left[\frac{\partial^{|\mathcal{I}|} f(x)}{\partial x_{i_1} \partial x_{i_2}...\partial x_{i_{|\mathcal{I}|}}}\right] > 0$. However, the computational cost of evaluating higher-order derivatives scales exponentially with $|\mathcal{I}|$, rendering exhaustive search over large feature subsets impractical. We therefore focus on detecting pairwise interactions using the Hessian matrix, which captures second-order dependencies in the model and offers a practical solution for interaction detection.

\begin{definition}\label{definition:pairwise}
    \textnormal{(Pairwise Interaction)} For a neural network $f$, two features $x_i$ and $x_j$ exhibit feature attribution if the following holds true: $\mathbf{H}_{f_c} = \frac{\partial^2 f_c(\mathbf{x})}{\partial x_i \partial x_j}$ $>$ $\mu$, where $\mu$ is a hyperparameter.
\end{definition}

Each element of the Hessian $\mathbf{H}_{f_c}$ quantifies the second-order dependency between features, measuring how the joint perturbation of $x_i$ and $x_j$ influences the output $f_c(\mathbf{x})$. This is particularly useful in highly non-linear models like neural networks, where first-order derivatives may fail to capture inter-feature dependencies. The Hessian also naturally captures local model curvature, making it a faithful method for detecting feature interactions \cite{singla2019understanding}.

While effective for identifying interactions, the Hessian is not suitable for assigning importance scores to them. This limitation stems from gradient saturation, regions of the input space where the model becomes locally flat and unresponsive to input perturbations, leading to vanishing higher-order derivatives. As a result, the Hessian may underestimate the strength of interactions in these regions. To mitigate this, we restrict the Hessian’s role to interaction detection.

\paragraph{Smoothing ReLU Activation.}Prior work, such as Integrated Hessians~\cite{ih}, replaces ReLU with SoftPlus to enable higher-order differentiation. However, we found that SoftPlus introduces noisy interaction patterns. Instead, we adopt a smoother approximation proposed by Zhang \etal~\cite{fire}, which preserves ReLU-like behavior while enabling stable Hessian computation. Specifically, we replace $r(z) = \max(0, z)$ with the following differentiable function:

\begin{equation}
    h(z) =
    \begin{cases}
    \left( z + \sqrt{z^2 + \tau} \right)' = 1 + \frac{z}{\sqrt{z^2 + \tau}} & (z < 0) \\
    \left( \sqrt{z^2 + \tau} \right)' = \frac{z}{\sqrt{z^2 + \tau}} & (z \geq 0)
    \end{cases}
\end{equation}

where $\tau > 0$ is a smoothing hyperparameter. See Appendix~\ref{appendix:smoothingrelu} for plots of $h(z)$.

\subsubsection{Higher-Order Feature Interactions}\label{section:higherorderfeature}
To identify higher-order interaction sets  $|\mathcal{I}| > 2$, we build on pairwise Hessian-based interactions by recursively expanding them into larger sets. This is motivated by the observation that if a higher-order interaction exists, then all of its subsets must also exhibit interactions~\cite{higherorder}. We use a breadth-first strategy subject to semantic and size constraints to merge strongly interacting pairs into higher-order feature sets.

\textbf{Hyperparameter:} Due to the high dimensionality in images, we introduce two hyperparameters to control the size and relevance of interaction sets. First, we introduce a threshold $\mu$ on the pairwise feature interactions, which determines which pairwise Hessian entries qualify as interactions. This ensures that only strongly coupled features are merged into sets, reducing spurious connections. This hyperparameter is derived from Definition~\ref{definition:pairwise}. We also introduce a threshold $\nu$ which limits the number of features in any interaction set. This constraint is inspired by the concept of \textit{order of explanation} from Shapley-Taylor indices~\cite{taylor}, which bounds the number of features considered in interaction-based explanations. We provide ablation study on the impact of these hyperparaemters in \cref{ablation:hyperparameter}.

\textbf{Starting Feature:} To initialize the search, we require seed feature $x_i$ from the input. While any input feature could be used as a starting point, a better starting feature should be both important to the model's prediction and spatially coherent with interpretable regions of the image. To achieve this, we propose to fix $x_i$ as the highest attributed feature by Integrated Gradients (IG)~\cite{ig} in the segmentation masks produced by Segment Anything Model (SAM) \cite{sam}. We note that, while by default we choose the top-IG pixel within a candidate region to encourage discovery of a starting point that is important to the model's prediction, this seed choice is modular. In \cref{sec:ablation-seeds} we ablate random seeds and report effects on quality of saliency maps. 

Similarly, SAM is used here as a replaceable spatial prior to encourage spatial coherence. SAM \emph{does not} influence gradients or attribution values; it only restricts merges to occur preferentially within the same region. Crucially, SAM is replaceable: in \cref{sec:ablations-sam} we compare SAM with QuickShift~\cite{zhang2020semantic} (and a \emph{no segmentation} variant) and quantify the impact on the quality of saliency maps. In addition, while H-Sets utilizes SAM as a spatial prior, it still remains sensitive to non-semantic model shortcuts. We provide a dedicated analysis on the DecoyMNIST~\cite{erion2021improving} in Appendix~\ref{appendix:decoy_analysis}.

\textbf{Procedure.} Once the seed feature $x_i$ is chosen, we compute the second-order gradients $\mathbf{H}_{f_c} = \frac{\partial^2 f_c(\mathbf{x})}{\partial x_i \partial x_j}$ for all $x_j \in \mathbf{x}$, identifying features that exhibit strong pairwise interactions with $x_i$ (i.e., Hessian values exceeding threshold $\mu$). These features form an ordered candidate set $X = \{x_j : \mathbf{H}_{f_c}[i,j] > \mu\}$, which we iteratively append to the interaction set $\mathcal{I}$, starting with $\mathcal{I} = {x_i}$, until $|\mathcal{I}| = \nu$ or no further valid additions exist. If $\mathcal{I}$ is incomplete, we repeat the procedure using the most strongly interacting feature in $X$ as a new seed. This recursive construction yields a final interaction set $\mathcal{I}$ grounded in both model relevance and visual coherence. We then iteratively repeat this procedure for masks in SAM to obtain the final set $\mathcal{S}$. We set the number of interaction sets as $|\mathcal{S}|=5$, similar to prior works~\cite{arch}. We provide ablation study on the interaction set count in \cref{ablation:interactionSet}. The algorithmic steps are provided in \cref{alg:hessian}. Also, see Appendix~\ref{appendix:higherorder} for a diagrammatic overview.

\subsection{Interaction Set Attribution: IDG-Vis}
With a set of meaningful feature interaction sets $\mathcal{S}$ from the input image $\mathbf{x}$, we discuss a game-theoretic approach to attribute each set. First, we instantiate the standard TU-game.

\textbf{Cooperative Game Theory.} Given an image classifier $f$ and an input image $\mathbf{x}$, we define a positive transferable utility (TU)-game as $(N,v)$, where $N$ players correspond to the number of features in the input image $\mathbf{x}$, and $v \colon 2^N \longrightarrow \mathbb{R}$ is the characteristic set function, attributing the payoff to a set of players forming a coalition independently.

While the characteristic function $v$ quantifies the total contribution of any subset of features, it does not distinguish whether this contribution is due to individual effects or interactions between features. Therefore, to explicitly measure the unique contribution of interactions, we turn to the Harsanyi dividend~\cite{harsanyi1982simplified}, which decomposes 
$v$ into additive components corresponding to individual and joint effects.

\begin{definition}{\textnormal{(Harsanyi Dividend)}} 
 Given characteristic function $v \colon 2^N \!\to\! \mathbb{R}$, the Harsanyi dividend for set $T \subseteq N$ is
 \begin{align}
     d(v,T) = \sum_{S \subseteq T} (-1)^{|T|-|S|} v(S)  
 \end{align}
\end{definition}

The Harsanyi dividend $d(v, T)$ captures the unique contribution of coalition $T$ to the total value of the grand coalition. This makes it a natural choice for attributing interaction effects, where we wish to quantify the joint influence of a group of features, excluding their individual roles. 

We instantiate the TU-game’s characteristic function $v$ with our set-level attribution $A$, where $A(\mathcal{I})$ is an attribution score assigned to interaction set $\mathcal{I}$. This enables dividend computation to compute interaction-specific attributions that satisfy desired theoretical properties. 

\paragraph{IDG-Vis.} We now describe how to attribute a scalar importance score to each feature interaction set $\mathcal{I} \in \mathcal{S}$. We introduce IDG-Vis (Integrated Directional Gradients for Vision), an adaptation of the original IDG framework~\cite{idg}, which was developed for text models. IDG-Vis extends this approach to the image domain by constructing directional vectors over pixels for set-based interaction attribution. 

\textbf{Motivation:} Given an interaction set $\mathcal{I} \subseteq \{1, \dots, d\}$, our goal is to assign a scalar value $a(\mathcal{I})$ that reflects the collective contribution of features in $\mathcal{I}$ to the model prediction $f(\mathbf{x})$. The key idea is to measure the change in the model's output in the direction of the interaction set. This directional signal provides a unified importance score that captures the joint influence of features in $\mathcal{I}$.

Let $\mathbf{a} \in \mathbb{R}^d$ be a direction vector constructed as:

\begin{minipage}[t]{0.48\textwidth}
\begin{equation}
a_i =
\begin{cases}
x_i - x'_i & \text{if } x_i \in \mathcal{I} \\
0 & \text{otherwise}
\end{cases}
\label{vec}
\end{equation}
\end{minipage}
\hfill
\begin{minipage}[t]{0.48\textwidth}
\begin{equation}
\hat{\mathbf{a}} = \frac{\mathbf{a}}{\|\mathbf{a}\|}
\label{dir}
\end{equation}
\end{minipage}

The directional gradient is then defined as the directional derivative along $\hat{\mathbf{a}}$, given by $\nabla_\mathcal{I} f(\mathbf{x}) = \left| \nabla f(\mathbf{x}) \cdot \hat{\mathbf{a}} \right| \label{grad}$. We take the absolute value to ensure compatibility with the positive TU-game framework discussed earlier. 

To reduce sensitivity to local gradient noise and saturation effects, we integrate the directional gradient along a linear path from a baseline $\mathbf{x}'$ to the input $\mathbf{x}$. 
\begin{equation}
\text{IDG-Vis}(T) = \int_{\alpha=0}^{1} \nabla_T f(\mathbf{x}' + \alpha(\mathbf{x} - \mathbf{x}')) \, d\alpha \label{idg}
\end{equation}
We define the set-level attribution score via Monte Carlo estimation of the Harsanyi dividend:
\begin{equation}
a(\mathcal{I}) \approx \sum_{T \sim \mathcal{U}(\mathcal{P}(\mathcal{I}))} \text{IDG-Vis}(T) \label{val}
\end{equation}
where $\mathcal{P}(\mathcal{I})$ is the power set of $\mathcal{I}$ and $\mathcal{U}(\cdot)$ denotes uniform sampling. This captures the interaction effect of $\mathcal{I}$ as the expected directional contribution over its subsets. To create a saliency map, we aggregate the attributions for each feature interaction set into one tensor. 

\textbf{Efficient Variant:} Since computing the directional gradient for all subsets of $\mathcal{I}$ is infeasible in high-dimensional spaces, we approximate the expected value in Equation~\eqref{val} using $t$ Monte Carlo samples. Additionally, we discretize the path integral in Equation~\eqref{idg} using a Riemann sum over $m$ steps:
\begin{gather}
    \text{IDG-Vis}(\mathcal{I}) = \frac{1}{m+1} \sum_{k=0}^m \nabla_\mathcal{I} f(\mathbf{x}' + \frac{k}{m}(\mathbf{x} - \mathbf{x}'))
\end{gather}

 We typically set $t, m \in [50, 100]$, following the recommendation of Integrated Gradients~\cite{ig}. We provide algorithmic steps for \textbf{IDG-Vis} in \cref{alg:dividends}.


\subsection{Axioms}\label{sec:axioms}
We align \textbf{IDG-Vis} with game-theoretic and attribution axioms to ensure predictable behavior of set attributions. Our attribution method satisfies:

\textbf{Axiom 1} (Non-Negativity) \textit{Every feature subset has a non-negative value $a(T) \geq 0.$}

\textbf{Axiom 2} (Normality) \textit{The value of the empty set of features is zero. $a(\emptyset) = 0$}

\textbf{Axiom 3} (Monotonicity) \textit{The value of a set of features is greater than or equal to the value of any of its subsets. If $R \subseteq T$, then $a(R) \leq a(T)$}

\textbf{Axiom 4} (Superadditivity) \textit{The value of the union of two disjoint sets of features is greater than or equal to the sum of the values of the two sets. If $R \cap T = \emptyset$, then $a(R \cup T) \geq a(R) + a(T).$}

These four axioms originate from classical game theory. Non-negativity ensures a positive TU-game; normality prevents attributing value to null interactions; monotonicity ensures larger coalitions receive at least as much value; and superadditivity implies that merging sets cannot reduce the total value. We now outline axioms prevalent in feature attribution methods \cite{ig}:

\textbf{Axiom 5} (Approximate Completeness)  \textit{The sum of $a(\mathcal{I})$ over $\mathcal{I} \in \mathcal{S}$ approximates the output difference between $f(\mathbf{x})$ and $f(\mathbf{x'})$. $\sum_{\mathcal{I} \in \mathcal{S}} a(\mathcal{I}) \leq f(\mathbf{x}) - f(\mathbf{x'})$ with equality if the interaction set $\mathcal{S}$ and its sampled subsets fully span the input difference vector $\mathbf{x} - \mathbf{x}'$.}

\textbf{Axiom 6} (Sensitivity) \textit{If an input $\mathbf{x}$ and a baseline $\mathbf{x'}$ are equal everywhere except $\mathbf{x}_\mathcal{I} \neq \mathbf{x'}_\mathcal{I}$ and if $f(\mathbf{x}) \neq f(\mathbf{x'})$, then $a(\mathcal{I}) \neq 0$} 

\textbf{Axiom 7} (Implementation Invariance) \textit{Two neural networks $f(\cdot)$ and $f'(\cdot)$, with corresponding value functions $a'$ and $a''$, are functionally equivalent if $f'(\mathbf{x}) = f''(\mathbf{x})$ for all $\mathbf{x}$. Then, $a'(\mathcal{I}) = a''(\mathcal{I})$ for sets $\mathcal{I} \in \mathcal{S}$}.

\textbf{Axiom 8} (Linearity) \textit{Given two neural networks $f'(\cdot)$ and $f''(\cdot)$, and $f(\mathbf{x}) = c \cdot f'(\mathbf{x}) + d \cdot f''(\mathbf{x})$, then $a(\mathcal{I}) = c \cdot a'(\mathcal{I}) + d \cdot a''(\mathcal{I})$}

\textbf{Axiom 9} (Symmetry-Preserving) \textit{Let $\mathcal{I}_1$ and $\mathcal{I}_2$ be symmetric interaction sets with equal cardinality. If swapping features in $\mathcal{I}_1$ and $\mathcal{I}_2$ does not change the function $f$ then $a(\mathcal{I}_1) = a(\mathcal{I}_2)$.} 

These standard attribution axioms make results comparable across inputs and models. \emph{Approximate completeness} ties the total set value to the logit difference $f_c(\mathbf{x})-f_c(\mathbf{x}')$; \emph{sensitivity} enforces responsiveness—if only features in $\mathcal{I}$ change and the prediction changes, then $a(\mathcal{I})\!\neq\!0$; \emph{implementation invariance} ensures reparameterizations of the same function yield identical attributions; \emph{linearity} preserves proportionality for mixtures; and \emph{symmetry} prevents arbitrary preference among exchangeable sets.

All of these essential axioms are satisfied by our approach. We discuss the proofs in \cref{appendix:proof}. 

\section{Experiments}\label{section:experiments}
\subsection{Setup}
\textbf{Dataset and Models:} We perform our evaluation on the ImageNet \cite{deng2009imagenet} and CUB~\cite{wah2011caltech} validation set using VGG16~\cite{vgg}, ResNet101~\cite{resnet}, DenseNet121~\cite{densenet}, and MobileNetV3~\cite{mobilenet}.
\newline 
\noindent 
\textbf{Baselines:} We compare our method against non-interaction-based attribution method, Integrated Gradients (IG) \cite{ig}, and existing interaction-based attribution methods: context-aware first-order (CAFO) second-order explanations (CASO) explanations~\cite{caso}, Archipelago (Arch) \cite{arch} and MOXI~\cite{moxi}. We discuss these methods in detail in Appendix~\ref{appendix:baseline}.  We evaluate with more baselines in Appendix~\ref{appendix:morebaseline}.
\newline 
\noindent 
\textbf{Metrics:} We evaluate quality of explanations using sparsity and faithfulness. Sparsity, measured with Gini index~\cite{chalasani2020concise}, evaluates the concentration of importance scores to assess the comprehensibility of explanations. Faithfulness, measured with {ROAD}~\cite{rong2022consistent}, quantifies how true the explanations are to the underlying model. ROAD measures how model accuracy degrades when the most important features are replaced using a noisy linear imputation technique. We opted for ROAD over Insertion/Deletion \cite{petsiukrise} and ROAR \cite{hooker2019benchmark} because Insertion/Deletion introduces artifacts, leading to distribution shifts in perturbed inputs, and ROAR requires costly model retraining. We report $\text{ROAD}_\text{AOPC}$, the area over the perturbation curve, where higher scores indicate more faithful attributions. All results are averaged over 1000 correctly classified samples. Full metric details are provided in Appendix~\ref{appendix:metrics}. We evaluate with additional faithfulness metric in Appendix~\ref{appendix:moremetrics}.

\noindent 
{\textbf{Hyperparameters:}} We set max elements in an interaction set ($\nu$) to 2000 features for ImageNet, and 3000 features for CUB. We set the threshold for Hessian ($\mu$) to 0.5 for both ImageNet and CUB. We discuss ablation for hyperparameters in \cref{ablation:hyperparameter}. We set the number of feature interaction sets $|\mathcal{S}|$ to 5, similar to Archipelago \cite{arch} (ablation in \cref{ablation:interactionSet} ). 

\subsection{Results}\label{section:results}
\paragraph{Qualitative analysis.} Figure~\ref{fig:qual-analy} compares saliency maps generated by existing attribution methods with H-Sets on representative samples from ImageNet and CUB. Integrated Gradients (IG)~\cite{ig} yields diffuse and noisy attributions, spreading importance across the input image. Archipelago produces coarse, mask-like regions that obscure finer details by treating entire segments uniformly. CAFO and CASO~\cite{caso} yield sharper maps but are still denser. MOXI~\cite{moxi} produces extremely compact maps, yet these activations are often too localized—focusing on single high-response regions while missing the broader spatial context necessary for interpretation. In contrast, \textbf{H-Sets} delineates compact and semantically coherent regions that correspond to discriminative object parts capturing the model’s decision basis with high spatial precision. More examples are provided in Appendix~\ref{appendix:qualitativeexamples}.

\paragraph{Sparsity.} Table~\ref{tab:sparsity} reports the Gini index scores to quantify the sparsity of explanations across different explanation methods and architectures on ImageNet and CUB datasets. A higher Gini score indicates that the explanation highlights fewer, more decisive features while suppressing noise from irrelevant regions. Across both ImageNet and CUB, \textbf{H-Sets} consistently achieves the highest sparsity, outperforming the baselines. {MOXI}~\cite{moxi} also attains strong sparsity, slightly exceeding \textbf{H-Sets} on select architectures.

\begin{table}[t]
\centering
\caption{
Sparsity comparison (Gini index; higher is better) of our method (H-Sets) versus  Integrated Gradients (IG)~\cite{ig}, context-aware first-order~(CAFO), and second-order explanations(CASO)~\cite{caso}, Archipelago (Arch)~\cite{arch} and MOXI~\cite{moxi}, using Gini-index. Values are mean $\pm$ std for 5 runs.}
\label{tab:sparsity}
\resizebox{0.45\textwidth}{!}{
\begin{tabular}{l l c c c c c c}
\toprule
\multicolumn{2}{c}{} &
\multicolumn{6}{c}{\textbf{Explanation Methods}} \\
\cmidrule(l){3-8}
\textbf{Dataset} & \textbf{Model} & IG & Arch & CAFO & CASO & MoXI & \textbf{H-Sets} \\
\midrule

\multirow{4}{*}{\textbf{ImageNet}} 
& VGG        & 0.71{\scriptsize$\pm$0.08} & 0.91{\scriptsize$\pm$0.03} & 0.83{\scriptsize$\pm$0.05} & 0.84{\scriptsize$\pm$0.04} & 0.85{\scriptsize$\pm$0.07} & \textbf{0.95{\scriptsize$\pm$0.01}} \\
& ResNet     & 0.81{\scriptsize$\pm$0.11} & 0.91{\scriptsize$\pm$0.03} & 0.84{\scriptsize$\pm$0.11} & 0.83{\scriptsize$\pm$0.11} & 0.87{\scriptsize$\pm$0.07} & \textbf{0.98{\scriptsize$\pm$0.01}} \\
& DenseNet   & 0.60{\scriptsize$\pm$0.05} & 0.90{\scriptsize$\pm$0.04} & 0.73{\scriptsize$\pm$0.05} & 0.74{\scriptsize$\pm$0.05} & 0.72{\scriptsize$\pm$0.14} & \textbf{0.94{\scriptsize$\pm$0.01}} \\
& MobileNet  & 0.63{\scriptsize$\pm$0.05} & 0.91{\scriptsize$\pm$0.05} & 0.79{\scriptsize$\pm$0.04} & 0.80{\scriptsize$\pm$0.04} & \textbf{0.96{\scriptsize$\pm$0.01}} & 0.95{\scriptsize$\pm$0.01} \\
\midrule

\multirow{4}{*}{\textbf{CUB}} 
& VGG        & 0.67{\scriptsize$\pm$0.08} & 0.92{\scriptsize$\pm$0.03} & 0.64{\scriptsize$\pm$0.13} & 0.49{\scriptsize$\pm$0.29} & 0.91{\scriptsize$\pm$0.08} & \textbf{0.93{\scriptsize$\pm$0.02}} \\
& ResNet     & 0.59{\scriptsize$\pm$0.07} & 0.89{\scriptsize$\pm$0.04} & 0.55{\scriptsize$\pm$0.08} & 0.56{\scriptsize$\pm$0.07} & 0.87{\scriptsize$\pm$0.13} & \textbf{0.92{\scriptsize$\pm$0.02}} \\
& DenseNet   & 0.57{\scriptsize$\pm$0.07} & 0.89{\scriptsize$\pm$0.05} & 0.52{\scriptsize$\pm$0.12} & 0.48{\scriptsize$\pm$0.18} & \textbf{0.93{\scriptsize$\pm$0.07}} & 0.91{\scriptsize$\pm$0.02} \\
& MobileNet  & 0.58{\scriptsize$\pm$0.06} & 0.89{\scriptsize$\pm$0.05} & 0.48{\scriptsize$\pm$0.04} & 0.52{\scriptsize$\pm$0.08} & 0.81{\scriptsize$\pm$0.13} & \textbf{0.89{\scriptsize$\pm$0.02}} \\
\bottomrule
\end{tabular}
}
\end{table}

\paragraph{Faithfulness.} Table~\ref{tab:faithfulnesseval} presents the faithfulness evaluation using $\text{ROAD}_\text{AOPC}$ scores, measuring how closely each explanation aligns with the model’s true decision behavior. Higher values indicate that removing the most attributed regions leads to a steeper drop in accuracy, reflecting more faithful explanations. Across both ImageNet and CUB, \textbf{H-Sets} consistently achieves the highest faithfulness scores on every architecture. Notably, {MOXI}~\cite{moxi} again performs competitively on certain ImageNet architectures. Associated ROAD plots are provided in Appendix~\ref{appendix:ROADplots}.

\begin{table}[t]
\centering
\caption{
Faithfulness comparison ($\mathrm{ROAD}_{\mathrm{AOPC}}$; higher is better) of our method (H-Sets) 
versus Integrated Gradients (IG)~\cite{ig}, context-aware first-order (CAFO), and second-order explanations(CASO)~\cite{caso}, Archipelago (Arch)~\cite{arch} and MOXI~\cite{moxi}, using $\text{ROAD}_\text{AOPC}$ score. Values are mean $\pm$ std for 5 runs.
}
\label{tab:faithfulnesseval}
\resizebox{0.48\textwidth}{!}{
\begin{tabular}{l l c c c c c c}
\toprule
\multicolumn{2}{c}{} &
\multicolumn{6}{c}{\textbf{Explanation Methods}} \\
\cmidrule(l){3-8}
\textbf{Dataset} & \textbf{Model} & IG & Arch & CAFO & CASO & MoXI & \textbf{H-Sets} \\
\midrule

\multirow{4}{*}{\textbf{ImageNet}}
& VGG        & 0.13{\scriptsize$\pm$0.01} & 0.27{\scriptsize$\pm$0.02} & 0.12{\scriptsize$\pm$0.01} & 0.12{\scriptsize$\pm$0.03} & 0.32{\scriptsize$\pm$0.03} & \textbf{0.34{\scriptsize$\pm$0.01}} \\
& ResNet     & 0.23{\scriptsize$\pm$0.01} & 0.21{\scriptsize$\pm$0.02} & 0.19{\scriptsize$\pm$0.00} & 0.19{\scriptsize$\pm$0.01} & 0.19{\scriptsize$\pm$0.01} & \textbf{0.26{\scriptsize$\pm$0.02}} \\
& DenseNet   & 0.30{\scriptsize$\pm$0.01} & 0.32{\scriptsize$\pm$0.01} & 0.25{\scriptsize$\pm$0.01} & 0.25{\scriptsize$\pm$0.01} & 0.35{\scriptsize$\pm$0.03} & \textbf{0.38{\scriptsize$\pm$0.01}} \\
& MobileNet  & 0.06{\scriptsize$\pm$0.01} & 0.24{\scriptsize$\pm$0.01} & 0.26{\scriptsize$\pm$0.02} & 0.26{\scriptsize$\pm$0.01} & 0.33{\scriptsize$\pm$0.03} & \textbf{0.37{\scriptsize$\pm$0.01}} \\
\midrule

\multirow{4}{*}{\textbf{CUB}}
& VGG        & 0.61{\scriptsize$\pm$0.01} & 0.06{\scriptsize$\pm$0.00} & 0.44{\scriptsize$\pm$0.01} & 0.44{\scriptsize$\pm$0.00} & 0.58{\scriptsize$\pm$0.01} & \textbf{0.65{\scriptsize$\pm$0.00}} \\
& ResNet     & 0.61{\scriptsize$\pm$0.01} & 0.55{\scriptsize$\pm$0.01} & 0.41{\scriptsize$\pm$0.01} & 0.41{\scriptsize$\pm$0.01} & 0.58{\scriptsize$\pm$0.01} & \textbf{0.63{\scriptsize$\pm$0.01}} \\
& DenseNet   & 0.60{\scriptsize$\pm$0.00} & 0.49{\scriptsize$\pm$0.00} & 0.51{\scriptsize$\pm$0.01} & 0.51{\scriptsize$\pm$0.01} & 0.55{\scriptsize$\pm$0.01} & \textbf{0.65{\scriptsize$\pm$0.00}} \\
& MobileNet  & 0.56{\scriptsize$\pm$0.01} & 0.20{\scriptsize$\pm$0.01} & 0.40{\scriptsize$\pm$0.01} & 0.40{\scriptsize$\pm$0.01} & 0.59{\scriptsize$\pm$0.01} & \textbf{0.60{\scriptsize$\pm$0.00}} \\
\bottomrule
\end{tabular}
}
\end{table}

\section{Ablation: Hyperparameter}\label{ablation:hyperparameter}
\paragraph{Number of features.} Table~\ref{tab:ablation-imagenet-side-by-side} (left) and Figure~\ref{fig:numfeatures-imagenet} illustrate how varying the number of features $\nu$ affects \textbf{H-Sets} explanation on ImageNet. When $\nu$ is small (e.g., 250), the resulting saliency maps are extremely sparse, capturing only the most dominant features. 
As $\nu$ increases, the maps gradually expand to cover richer visual structures, improving contextual completeness but also reducing sparsity. Interestingly, the $\text{ROAD}_{\text{AOPC}}$ score remains largely stable across this range, indicating that faithfulness is preserved even with fewer features, suggesting that a relatively small subset of high-interaction features suffices to explain the model’s predictions. 
However, larger $\nu$ increases computational cost since each added feature introduces additional Hessian evaluations and set-level directional gradients. 
We therefore adopt $\nu = 2000$ as a default, which provides visually coherent maps while maintaining efficiency.

\begin{table}[t]
\centering
\caption{
Ablation of H-Sets hyperparameters on ImageNet (MobileNet). \textbf{Left:} varying the number of features included in explanations ($\nu$). \textbf{Right:} varying the Hessian interaction threshold ($\mu$).
}
\label{tab:ablation-imagenet-side-by-side}
\vspace{4pt}

\begin{minipage}{0.45\linewidth}
\centering
\textbf{(a) Feature Count Ablation ($\nu$)}
\vspace{4pt}

\resizebox{\textwidth}{!}{
\begin{tabular}{l c c}
\toprule
\textbf{\# Features ($\nu$)} & \textbf{Sparsity} & $\mathbf{ROAD}_{\mathbf{AOPC}}$ \\
\midrule
250  & 0.99 & 0.42 \\
1000 & 0.97 & 0.40 \\
2000 & 0.94 & 0.37 \\
3000 & 0.90 & 0.39 \\
5000 & 0.85 & 0.37 \\
\bottomrule
\end{tabular}
}
\end{minipage}
\hspace{1.5em}
\begin{minipage}{0.45\linewidth}
\centering
\textbf{(b) Hessian Threshold Ablation ($\mu$)}
\vspace{4pt}

\resizebox{\textwidth}{!}{
\begin{tabular}{l c c}
\toprule
\textbf{Threshold ($\mu$)} & \textbf{Sparsity} & $\mathbf{ROAD}_{\mathbf{AOPC}}$ \\
\midrule
0.1 & 0.95 & 0.39 \\
0.2 & 0.95 & 0.37 \\
0.3 & 0.95 & 0.38 \\
0.4 & 0.96 & 0.38 \\
0.5 & 0.96 & 0.37 \\
0.6 & 0.96 & 0.41 \\
0.7 & 0.96 & 0.39 \\
0.8 & 0.96 & 0.39 \\
\bottomrule
\end{tabular}
}
\end{minipage}

\end{table}

\begin{figure}[ht]
    \centering
    \includegraphics[width=0.90\linewidth]{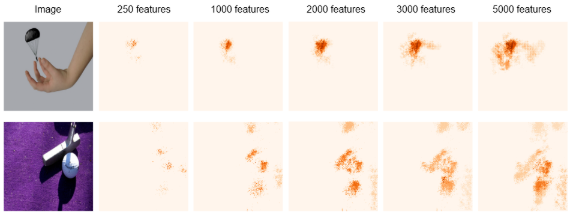}
    \caption{Saliency maps with different values of $\nu$ on ImageNet.}
    \label{fig:numfeatures-imagenet}
\end{figure}

\paragraph{Interaction threshold.} Table~\ref{tab:ablation-imagenet-side-by-side} (right) shows the effect of varying the Hessian threshold $\mu$ while fixing $\nu = 2000$. Both sparsity and faithfulness remain remarkably stable across a wide range of $\mu$, suggesting that \textbf{H-Sets} is robust to this hyperparameter and reliably detects semantically meaningful interactions without being overly sensitive to the exact strength of second-order gradients. Lower thresholds (\(\mu < 0.3\)) admit more feature pairs, resulting in slightly denser maps, whereas higher thresholds yield more selective interactions but only marginal changes in $\text{ROAD}_{\text{AOPC}}$.

Similar trends are observed in ablation results on the CUB dataset in Appendix~\ref{appendix:ablationHyperparameterCUB}. 
\section{Ablation: Spatial Prior}\label{sec:ablations-sam}

\cref{tab:ablate_segmentation_road} and~\cref{tab:ablate_segmentation_sparsity} analyze the influence of spatial priors used to initialize seed regions during interaction discovery. 
We compare three variants: \emph{SAM~\cite{sam}+IG~\cite{ig}}, \emph{}{Quickshift~\cite{zhang2020semantic}+IG~\cite{ig}}, and a baseline \emph{IG~\cite{ig} without segmentation}. 
\cref{tab:ablate_segmentation_road} shows that across both ImageNet and CUB, using SAM as a spatial prior yields the highest or comparable $\text{ROAD}_{\text{AOPC}}$ scores, demonstrating that accurate mask initialization improves attribution faithfulness. The gains are most pronounced on fine-grained CUB classes, where SAM’s high-quality, part-aware masks guide the interaction discovery process toward semantically consistent regions, avoiding spurious context such as background textures. 

Interestingly, sparsity trends reveal a complementary pattern in \cref{tab:ablate_segmentation_sparsity}. While the \emph{no-segmentation} variant occasionally achieves slightly higher Gini scores, its explanations are less faithful to the model’s reasoning, indicating that over-sparse maps without a spatial prior may discard relevant evidence. \emph{SAM}~\cite{sam} as a segmentation prior provides the best balance between faithfulness and sparsity.
\begin{table}[t]
\centering
\caption{
Ablation study of segmentation strategies on the faithfulness metric (higher is better). 
Values are mean $\pm$ std over 5 runs.
}
\label{tab:ablate_segmentation_road}
\resizebox{0.48\textwidth}{!}{
\begin{tabular}{l l c c c}
\toprule
\multicolumn{2}{c}{} &
\multicolumn{3}{c}{\textbf{Segmentation Methods}} \\
\cmidrule(l){3-5}
\textbf{Dataset} & \textbf{Model} & SAM+IG & Quickshift+IG & No SAM+IG \\
\midrule

\multirow{4}{*}{\textbf{ImageNet}}
& VGG         & 0.34{\scriptsize$\pm$0.01} & 0.34{\scriptsize$\pm$0.02} & 0.33{\scriptsize$\pm$0.02} \\
& ResNet   & 0.26{\scriptsize$\pm$0.02} & 0.21{\scriptsize$\pm$0.03} & 0.18{\scriptsize$\pm$0.01} \\
& DenseNet & 0.38{\scriptsize$\pm$0.01} & 0.38{\scriptsize$\pm$0.01} & 0.39{\scriptsize$\pm$0.03} \\
& MobileNet & 0.37{\scriptsize$\pm$0.01} & 0.38{\scriptsize$\pm$0.01} & 0.37{\scriptsize$\pm$0.02} \\
\midrule

\multirow{4}{*}{\textbf{CUB}}
& VGG         & 0.65{\scriptsize$\pm$0.00} & 0.62{\scriptsize$\pm$0.01} & 0.59{\scriptsize$\pm$0.02} \\
& ResNet   & 0.63{\scriptsize$\pm$0.01} & 0.59{\scriptsize$\pm$0.03} & 0.59{\scriptsize$\pm$0.02} \\
& DenseNet & 0.65{\scriptsize$\pm$0.00} & 0.58{\scriptsize$\pm$0.02} & 0.56{\scriptsize$\pm$0.03} \\
& MobileNet & 0.60{\scriptsize$\pm$0.00} & 0.62{\scriptsize$\pm$0.01} & 0.62{\scriptsize$\pm$0.01} \\
\bottomrule
\end{tabular}
}
\end{table}

\begin{table}[t]
\centering
\caption{
Ablation study of segmentation strategies on the Sparsity metric (higher is better). 
Values are mean $\pm$ std over 5 runs.
}
\label{tab:ablate_segmentation_sparsity}
\resizebox{0.48\textwidth}{!}{
\begin{tabular}{l l c c c}
\toprule
\multicolumn{2}{c}{} &
\multicolumn{3}{c}{\textbf{Segmentation Methods}} \\
\cmidrule(l){3-5}
\textbf{Dataset} & \textbf{Model} & SAM+IG & Quickshift+IG & No SAM+IG \\
\midrule

\multirow{4}{*}{\textbf{ImageNet}}
& VGG         & 0.95{\scriptsize$\pm$0.01} & 0.95{\scriptsize$\pm$0.01} & 0.97{\scriptsize$\pm$0.01} \\
& ResNet   & 0.98{\scriptsize$\pm$0.01} & 0.95{\scriptsize$\pm$0.01} & 0.96{\scriptsize$\pm$0.01} \\
& DenseNet & 0.94{\scriptsize$\pm$0.01} & 0.96{\scriptsize$\pm$0.01} & 0.96{\scriptsize$\pm$0.01} \\
& MobileNet & 0.95{\scriptsize$\pm$0.01} & 0.94{\scriptsize$\pm$0.01} & 0.94{\scriptsize$\pm$0.02} \\
\midrule

\multirow{4}{*}{\textbf{CUB}}
& VGG         & 0.93{\scriptsize$\pm$0.02} & 0.95{\scriptsize$\pm$0.01} & 0.96{\scriptsize$\pm$0.01} \\
& ResNet   & 0.92{\scriptsize$\pm$0.02} & 0.94{\scriptsize$\pm$0.02} & 0.96{\scriptsize$\pm$0.01} \\
& DenseNet & 0.91{\scriptsize$\pm$0.02} & 0.94{\scriptsize$\pm$0.02} & 0.96{\scriptsize$\pm$0.01} \\
& MobileNet & 0.89{\scriptsize$\pm$0.02} & 0.93{\scriptsize$\pm$0.02} & 0.94{\scriptsize$\pm$0.02} \\
\bottomrule
\end{tabular}
}
\end{table}

\subsection{Ablation: Seed selection}\label{sec:ablation-seeds}

\cref{tab:ablate_seed_road} and \cref{tab:ablate_seed_sparsity} analyze how the choice of seed initialization influences interaction discovery and resulting attributions. 
We compare three configurations: (\textbf{i}) random seeds with SAM segmentation, (\textbf{ii}) IG-based seeds within SAM masks (ours), and (\textbf{iii}) random seeds without segmentation. \cref{tab:ablate_seed_road} shows that across both ImageNet and CUB, initializing from the top-IG feature inside each SAM region yields consistently higher $\text{ROAD}_{\text{AOPC}}$ scores, confirming that seeding interaction search with features already deemed salient by the model leads to more causally aligned explanations. This initialization effectively focuses the Hessian-based search on high-relevance regions, avoiding exploration of arbitrary local dependencies or background noise.

\cref{tab:ablate_seed_sparsity}, sparsity remains stable across all configurations, showing that the improvement in faithfulness does not compromise explanation conciseness. Interestingly, combining IG seeding with SAM achieves the strongest performance, suggesting that the combination of saliency-informed initialization and spatial priors helps \textbf{H-Sets} converge on the most semantically coherent interactions. 

\begin{table}[t]
\centering
\caption{
Ablation of seed initialization strategies on the faithfulness $\mathrm{ROAD}$ metric (higher is better). 
Values are mean $\pm$ std over 5 runs.
}
\label{tab:ablate_seed_road}
\resizebox{0.48\textwidth}{!}{
\begin{tabular}{l l c c c}
\toprule
\multicolumn{2}{c}{} &
\multicolumn{3}{c}{\textbf{Seed Initialization Methods}} \\
\cmidrule(l){3-5}
\textbf{Dataset} & \textbf{Model} 
& Random {+SAM} 
& IG {+SAM} 
& Random {+No SAM} \\
\midrule

\multirow{4}{*}{\textbf{ImageNet}}
& VGG         & 0.30{\scriptsize$\pm$0.05} & 0.34{\scriptsize$\pm$0.01} & 0.30{\scriptsize$\pm$0.05} \\
& ResNet   & 0.17{\scriptsize$\pm$0.05} & 0.26{\scriptsize$\pm$0.02} & 0.17{\scriptsize$\pm$0.06} \\
& DenseNet & 0.30{\scriptsize$\pm$0.04} & 0.38{\scriptsize$\pm$0.01} & 0.31{\scriptsize$\pm$0.06} \\
& MobileNet & 0.32{\scriptsize$\pm$0.04} & 0.37{\scriptsize$\pm$0.01} & 0.31{\scriptsize$\pm$0.06} \\
\midrule

\multirow{4}{*}{\textbf{CUB}}
& VGG         & 0.50{\scriptsize$\pm$0.04} & 0.65{\scriptsize$\pm$0.00} & 0.55{\scriptsize$\pm$0.08} \\
& ResNet   & 0.49{\scriptsize$\pm$0.02} & 0.63{\scriptsize$\pm$0.01} & 0.50{\scriptsize$\pm$0.04} \\
& DenseNet & 0.52{\scriptsize$\pm$0.03} & 0.65{\scriptsize$\pm$0.00} & 0.49{\scriptsize$\pm$0.08} \\
& MobileNet & 0.57{\scriptsize$\pm$0.02} & 0.60{\scriptsize$\pm$0.00} & 0.56{\scriptsize$\pm$0.02} \\
\bottomrule
\end{tabular}
}
\end{table}

\begin{table}[t]
\centering
\caption{
Ablation of seed initialization strategies on the Sparsity metric (higher is better). 
Values are mean $\pm$ std over 5 runs.
}
\label{tab:ablate_seed_sparsity}
\resizebox{0.48\textwidth}{!}{
\begin{tabular}{l l c c c}
\toprule
\multicolumn{2}{c}{} &
\multicolumn{3}{c}{\textbf{Seed Initialization Methods}} \\
\cmidrule(l){3-5}
\textbf{Dataset} & \textbf{Model} 
& Random {+SAM} 
& IG {+SAM} 
& Random {+No SAM} \\
\midrule

\multirow{4}{*}{\textbf{ImageNet}}
& VGG         & 0.93{\scriptsize$\pm$0.01} & 0.95{\scriptsize$\pm$0.01} & 0.93{\scriptsize$\pm$0.01} \\
& ResNet   & 0.93{\scriptsize$\pm$0.01} & 0.98{\scriptsize$\pm$0.01} & 0.93{\scriptsize$\pm$0.01} \\
& DenseNet & 0.94{\scriptsize$\pm$0.01} & 0.94{\scriptsize$\pm$0.01} & 0.94{\scriptsize$\pm$0.01} \\
& MobileNet & 0.92{\scriptsize$\pm$0.01} & 0.95{\scriptsize$\pm$0.01} & 0.92{\scriptsize$\pm$0.01} \\
\midrule

\multirow{4}{*}{\textbf{CUB}}
& VGG         & 0.91{\scriptsize$\pm$0.01} & 0.93{\scriptsize$\pm$0.02} & 0.90{\scriptsize$\pm$0.01} \\
& ResNet   & 0.91{\scriptsize$\pm$0.01} & 0.92{\scriptsize$\pm$0.02} & 0.91{\scriptsize$\pm$0.01} \\
& DenseNet & 0.91{\scriptsize$\pm$0.01} & 0.91{\scriptsize$\pm$0.02} & 0.91{\scriptsize$\pm$0.01} \\
& MobileNet & 0.89{\scriptsize$\pm$0.01} & 0.89{\scriptsize$\pm$0.02} & 0.89{\scriptsize$\pm$0.01} \\
\bottomrule
\end{tabular}
}
\end{table}

\subsection{Ablation: Interaction set}\label{ablation:interactionSet}

\cref{tab:ablate_sets_road} and \cref{tab:ablate_sets_sparsity} analyze how varying the number of interaction sets $|\mathcal{S}|$ influences explanation quality. Each set captures up to 2,000 highly interacting features for ImageNet and 3000 features for CUB, sorted by interaction strength. Because this construction already prioritizes the most influential interactions, adding more sets primarily introduces weaker feature groups that contribute marginally to the model’s output. As a result, in \cref{tab:ablate_sets_road}, the faithfulness score ($\text{ROAD}_{\text{AOPC}}$) improves slightly from one to five sets but shows diminishing returns beyond that point.

Sparsity exhibits a corresponding downward trend in \cref{tab:ablate_sets_sparsity} as more sets are added, reflecting that weaker interactions disperse attribution scores over less relevant regions. The configuration with five sets achieves the best trade-off between interpretability and coverage; we therefore fix $|\mathcal{S}|=5$ in all main experiments. In terms of runtime, computational cost scales approximately linearly with the number of interaction sets—e.g.,  for VGG on CUB, average runtime increases from 15.8 s (1 set) to 20.7 s (3 sets), 25.0 s (5 sets), 30.3 s (7 sets), and 35.6 s (9 sets).

\begin{table}[t]
\centering
\caption{
Ablation of the number of feature interaction sets on the faithfulness $\mathrm{ROAD}$ metric (higher is better). 
Values are mean $\pm$ std over 5 runs.
}
\label{tab:ablate_sets_road}
\resizebox{0.48\textwidth}{!}{
\begin{tabular}{l l c c c c c}
\toprule
\multicolumn{2}{c}{} &
\multicolumn{5}{c}{\textbf{Number of Interaction Sets}} \\
\cmidrule(l){3-7}
\textbf{Dataset} & \textbf{Model} 
& 1 Set & 3 Sets & 5 Sets & 7 Sets & 9 Sets \\
\midrule

\multirow{4}{*}{\textbf{ImageNet}}
& VGG         & 0.33{\scriptsize$\pm$0.02} & 0.34{\scriptsize$\pm$0.01} & 0.34{\scriptsize$\pm$0.01} & 0.33{\scriptsize$\pm$0.01} & 0.35{\scriptsize$\pm$0.03} \\
& ResNet   & 0.25{\scriptsize$\pm$0.02} & 0.26{\scriptsize$\pm$0.01} & 0.26{\scriptsize$\pm$0.02} & 0.25{\scriptsize$\pm$0.03} & 0.25{\scriptsize$\pm$0.02} \\
& DenseNet & 0.37{\scriptsize$\pm$0.01} & 0.37{\scriptsize$\pm$0.02} & 0.38{\scriptsize$\pm$0.01} & 0.38{\scriptsize$\pm$0.02} & 0.39{\scriptsize$\pm$0.03} \\
& MobileNet & 0.35{\scriptsize$\pm$0.02} & 0.36{\scriptsize$\pm$0.01} & 0.37{\scriptsize$\pm$0.01} & 0.37{\scriptsize$\pm$0.02} & 0.37{\scriptsize$\pm$0.01} \\
\midrule

\multirow{4}{*}{\textbf{CUB}}
& VGG         & 0.66{\scriptsize$\pm$0.02} & 0.65{\scriptsize$\pm$002} & 0.65{\scriptsize$\pm$0.00} & 0.66{\scriptsize$\pm$0.02} & 0.69{\scriptsize$\pm$0.04} \\
& ResNet   & 0.58{\scriptsize$\pm$0.02} & 0.62{\scriptsize$\pm$0.02} & 0.63{\scriptsize$\pm$0.01} & 0.61{\scriptsize$\pm$0.01} & 0.64{\scriptsize$\pm$0.02} \\
& DenseNet & 0.63{\scriptsize$\pm$0.01} & 0.65{\scriptsize$\pm$0.01} & 0.65{\scriptsize$\pm$0.00} & 0.66{\scriptsize$\pm$0.02} & 0.66{\scriptsize$\pm$0.03} \\
& MobileNet & 0.58{\scriptsize$\pm$0.02} & 0.60{\scriptsize$\pm$0.02} & 0.60{\scriptsize$\pm$0.02} & 0.59{\scriptsize$\pm$0.02} & 0.59{\scriptsize$\pm$0.01} \\
\bottomrule
\end{tabular}
}
\end{table}

\begin{table}[t]
\centering
\caption{
Ablation of the number of feature interaction sets on the Sparsity metric (higher is better). 
Values are mean $\pm$ std over 5 runs.
}
\label{tab:ablate_sets_sparsity}
\resizebox{0.48\textwidth}{!}{
\begin{tabular}{l l c c c c c}
\toprule
\multicolumn{2}{c}{} &
\multicolumn{5}{c}{\textbf{Number of Interaction Sets}} \\
\cmidrule(l){3-7}
\textbf{Dataset} & \textbf{Model}
& 1 Set & 3 Sets & 5 Sets & 7 Sets & 9 Sets \\
\midrule

\multirow{4}{*}{\textbf{ImageNet}}
& VGG         & 0.98{\scriptsize$\pm$0.00} & 0.96{\scriptsize$\pm$0.01} & 0.95{\scriptsize$\pm$0.01} & 0.93{\scriptsize$\pm$0.01} & 0.92{\scriptsize$\pm$0.02} \\
& ResNet   & 0.98{\scriptsize$\pm$0.00} & 0.96{\scriptsize$\pm$0.01} & 0.95{\scriptsize$\pm$0.01} & 0.93{\scriptsize$\pm$0.01} & 0.91{\scriptsize$\pm$0.02} \\
& DenseNet & 0.98{\scriptsize$\pm$0.00} & 0.96{\scriptsize$\pm$0.01} & 0.94{\scriptsize$\pm$0.01} & 0.93{\scriptsize$\pm$0.01} & 0.92{\scriptsize$\pm$0.02} \\
& MobileNet & 0.98{\scriptsize$\pm$0.00} & 0.95{\scriptsize$\pm$0.01} & 0.95{\scriptsize$\pm$0.01} & 0.90{\scriptsize$\pm$0.01} & 0.88{\scriptsize$\pm$0.02} \\
\midrule

\multirow{4}{*}{\textbf{CUB}}
& VGG         & 0.97{\scriptsize$\pm$0.00} & 0.95{\scriptsize$\pm$0.01} & 0.93{\scriptsize$\pm$0.01} & 0.93{\scriptsize$\pm$0.02} & 0.91{\scriptsize$\pm$0.02} \\
& ResNet   & 0.97{\scriptsize$\pm$0.00} & 0.95{\scriptsize$\pm$0.01} & 0.92{\scriptsize$\pm$0.02} & 0.91{\scriptsize$\pm$0.02} & 0.90{\scriptsize$\pm$0.02} \\
& DenseNet & 0.98{\scriptsize$\pm$0.00} & 0.95{\scriptsize$\pm$0.01} & 0.91{\scriptsize$\pm$0.02} & 0.91{\scriptsize$\pm$0.02} & 0.89{\scriptsize$\pm$0.02} \\
& MobileNet & 0.97{\scriptsize$\pm$0.00} & 0.94{\scriptsize$\pm$0.01} & 0.89{\scriptsize$\pm$0.02} & 0.90{\scriptsize$\pm$0.02} & 0.89{\scriptsize$\pm$0.03} \\
\bottomrule
\end{tabular}
}
\end{table}

\vspace{2mm}
\noindent
\textbf{Computational cost.} While the computation of second-order dependencies in \textbf{H-Sets} introduces an expected overhead, the runtime remains tractable and controllable via hyperparameter tuning. A comparative analysis of empirical runtimes are provided in Appendix~\ref{appendix:runtime}. 


\section{Limitations and Future Work}
While \textbf{H-Sets} is a principled framework for higher-order interactions, some limitations exist. First, adapting \textbf{H-Sets} to Vision Transformers (ViTs) is non-trivial due to their token-based representations. Exploring Hessian-guided interactions across self-attention blocks remains a key direction for future work. Second, performance depends on two hyperparameters—interaction threshold ($\mu$) and set size ($\nu$)—which may require light tuning across architectures. Finally, second-order derivatives introduce computational overhead compared to first-order methods. Although this cost is offset by higher faithfulness, future work will explore efficient Hessian approximations to further expand \textbf{H-Sets} applicability.

\section{Conclusion}
In this work, we propose \textbf{H-Sets}, a principled framework for discovering and attributing higher-order feature interactions in image classifiers. By leveraging Hessian-based interaction detection and a set-level attribution formulation grounded in cooperative game theory, \textbf{H-Sets} captures joint feature influences while satisfying key attribution axioms. Comprehensive experiments on ImageNet and CUB demonstrate that \textbf{H-Sets} produces explanations that are both sparser and more faithful than existing interaction-based methods. 

\section*{Acknowledgement}
This research was supported by Toyota-InfoTech Labs through Unrestricted Research Funds.
{
    \small
    \bibliographystyle{ieeenat_fullname}
    \bibliography{main}
}

\clearpage
\setcounter{page}{1}
\maketitlesupplementary

\appendix

\section*{Appendix}

\section{Computational cost}\label{appendix:runtime}
\begin{table}[h]
\centering
\caption{
Time complexity comparison of our method (H-Sets) versus Integrated Gradients (IG)~\cite{ig}, context-aware first-order (CAFO), and second-order explanations(CASO)~\cite{caso}, Archipelago (Arch)~\cite{arch} and MOXI~\cite{moxi}. Values are mean $\pm$ std in seconds over 5 runs.
}
\label{tab:time_complexity}
\resizebox{0.48\textwidth}{!}{
\begin{tabular}{l l c c c c c c}
\toprule
\multicolumn{2}{c}{} &
\multicolumn{6}{c}{\textbf{Explanation Methods}} \\
\cmidrule(l){3-8}
\textbf{Dataset} & \textbf{Model}
& IG & Arch & CAFO & CASO & MoXI & \textbf{H-Sets} \\
\midrule

\multirow{4}{*}{\textbf{ImageNet}}
& VGG       & 0.13{\scriptsize$\pm$0.02} & 9.44{\scriptsize$\pm$3.08} & 0.04{\scriptsize$\pm$0.01} & 2.44{\scriptsize$\pm$0.56} & 20.51{\scriptsize$\pm$3.06} & 23.58{\scriptsize$\pm$3.34} \\
& ResNet    & 0.19{\scriptsize$\pm$0.04} & 12.18{\scriptsize$\pm$3.36} & 0.17{\scriptsize$\pm$0.08} & 2.23{\scriptsize$\pm$0.55} & 16.35{\scriptsize$\pm$2.89} & 16.07{\scriptsize$\pm$0.93} \\
& DenseNet  & 0.25{\scriptsize$\pm$0.05} & 13.82{\scriptsize$\pm$4.82} & 0.08{\scriptsize$\pm$0.04} & 3.62{\scriptsize$\pm$0.82} & 18.00{\scriptsize$\pm$3.87} & 17.15{\scriptsize$\pm$1.09} \\
& MobileNet  & 0.09{\scriptsize$\pm$0.03} & 7.57{\scriptsize$\pm$2.39} & 0.13{\scriptsize$\pm$0.06} & 15.97{\scriptsize$\pm$5.81} & 9.96{\scriptsize$\pm$1.68} & 46.46{\scriptsize$\pm$5.54} \\
\midrule

\multirow{4}{*}{\textbf{CUB}}
& VGG          & 0.17{\scriptsize$\pm$0.04} & 9.92{\scriptsize$\pm$4.46} & 0.03{\scriptsize$\pm$0.02} & 2.09{\scriptsize$\pm$0.68} & 15.66{\scriptsize$\pm$3.20} & 25.00{\scriptsize$\pm$3.64} \\
& ResNet   & 0.08{\scriptsize$\pm$0.02} & 7.69{\scriptsize$\pm$3.09} & 0.06{\scriptsize$\pm$0.03} & 1.29{\scriptsize$\pm$0.29} & 16.35{\scriptsize$\pm$2.89} & 18.48{\scriptsize$\pm$0.93} \\
& DenseNet  & 0.27{\scriptsize$\pm$0.04} & 13.32{\scriptsize$\pm$3.74} & 0.17{\scriptsize$\pm$0.07} & 3.78{\scriptsize$\pm$0.53} & 15.66{\scriptsize$\pm$3.20} & 30.97{\scriptsize$\pm$8.53} \\
& MobileNet & 0.10{\scriptsize$\pm$0.09} & 5.92{\scriptsize$\pm$1.91} & 0.06{\scriptsize$\pm$0.04} & 3.55{\scriptsize$\pm$1.09} & 18.21{\scriptsize$\pm$16.82} & 17.39{\scriptsize$\pm$0.99} \\
\bottomrule
\end{tabular}
}
\end{table}

~\cref{tab:time_complexity} summarizes the average runtime per image (in seconds) across 1,000 validation samples. As expected, \textbf{H-Sets} requires more computation than first-order methods (e.g., IG) and some interaction-based methods. This overhead, however, is the direct result of computing second-order feature dependencies and enforcing attribution axioms (~\cref{sec:axioms}), which together yield more faithful explanations. Unlike region-masking approaches such as Arch and MOXI, H-Sets operates at a fine-grained, pixel-level resolution and produces mathematically interpretable interaction scores rather than heuristic saliency masks. 

Despite this, the runtime remains tractable and flexible. 
Both the number of features per interaction set ($\nu$) and the Hessian threshold ($\mu$) offer effective control over the accuracy–efficiency trade-off: smaller $\nu$ or higher $\mu$ values reduce runtime substantially with minimal effect on $\text{ROAD}_{\text{AOPC}}$ (see~\cref{tab:ablation-imagenet-side-by-side}). 

\section{Qualitative results}\label{appendix:qualitativeexamples}
Figure ~\ref{fig:appendix_b} show additional qualitative
comparisons of attribution methods across diverse ImageNet samples.
Each row shows the original image followed by saliency maps produced by different methods.

\begin{figure}[h]
 
    \includegraphics[width=\linewidth]{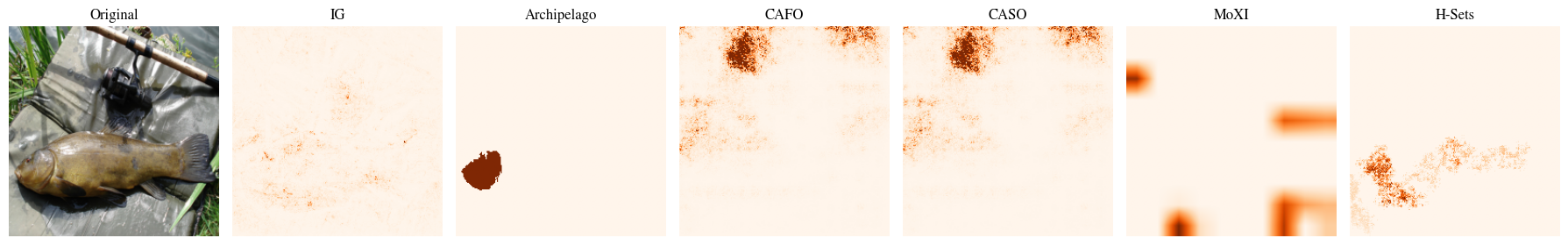}\\[3pt]
    \includegraphics[width=\linewidth]{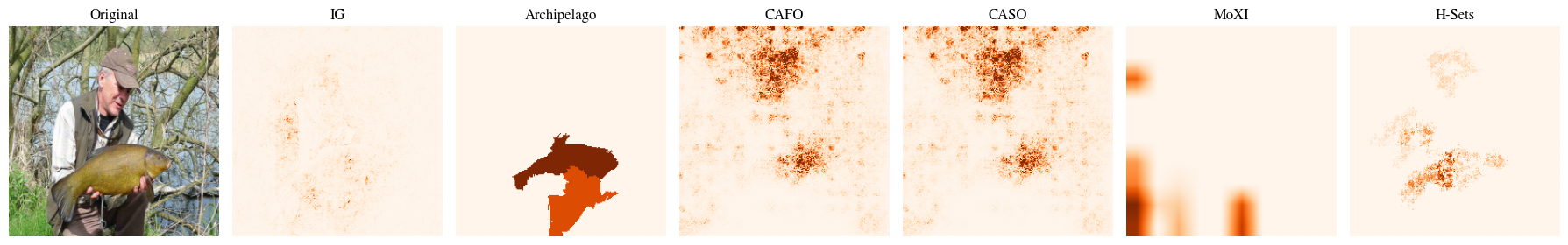}\\[3pt]
    \includegraphics[width=\linewidth]{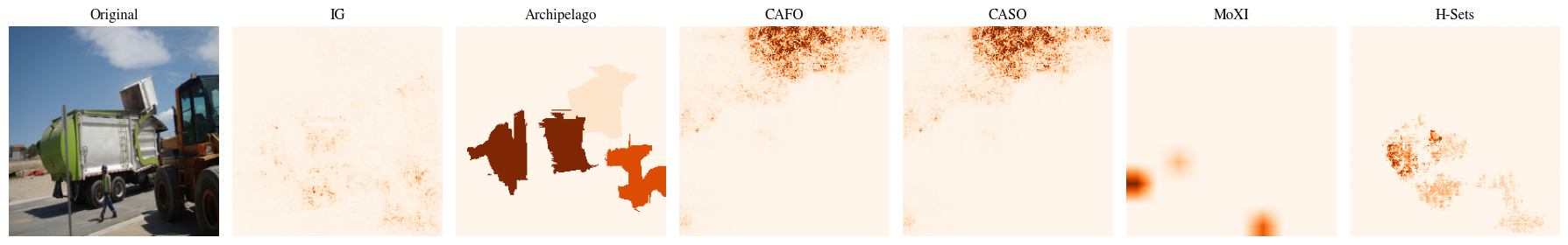}\\[3pt]
    \includegraphics[width=\linewidth]{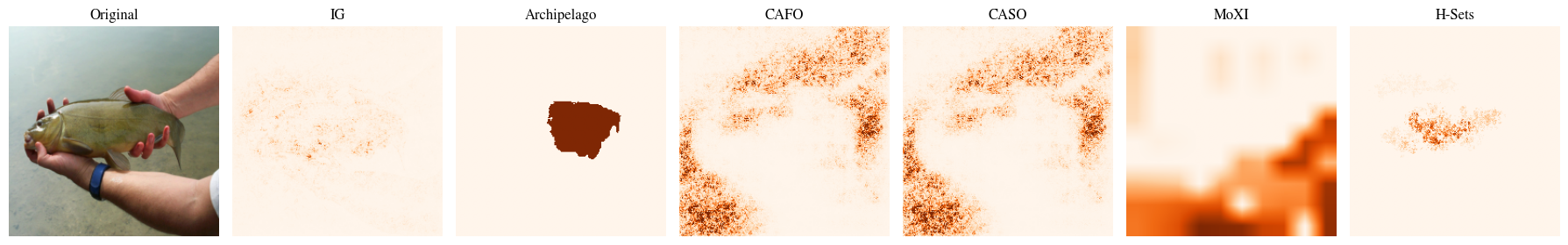}\\[3pt]
    \includegraphics[width=\linewidth]{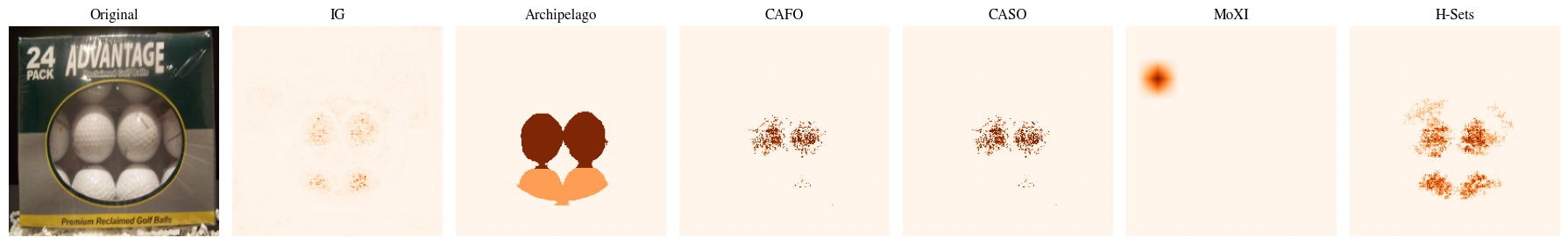}\\[3pt]
    \includegraphics[width=\linewidth]{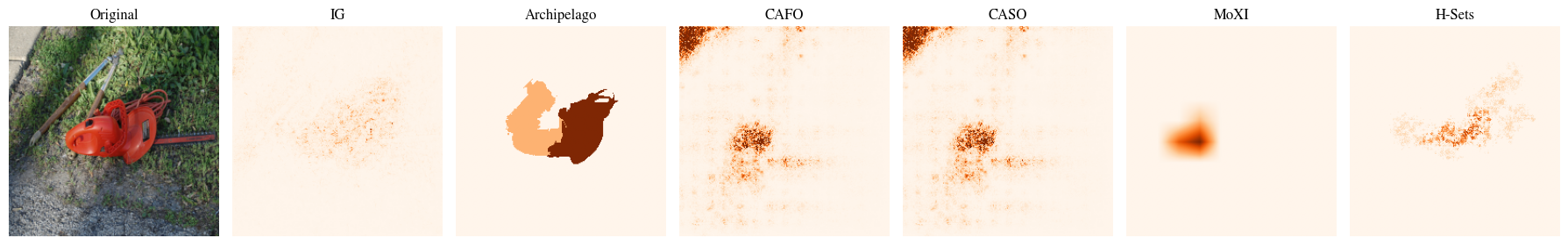}\\[3pt]
    \includegraphics[width=\linewidth]{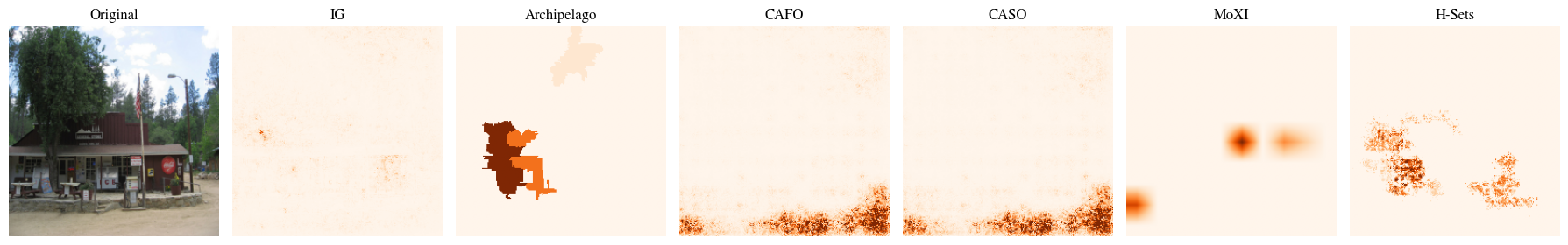}\\[3pt]
    \includegraphics[width=\linewidth]{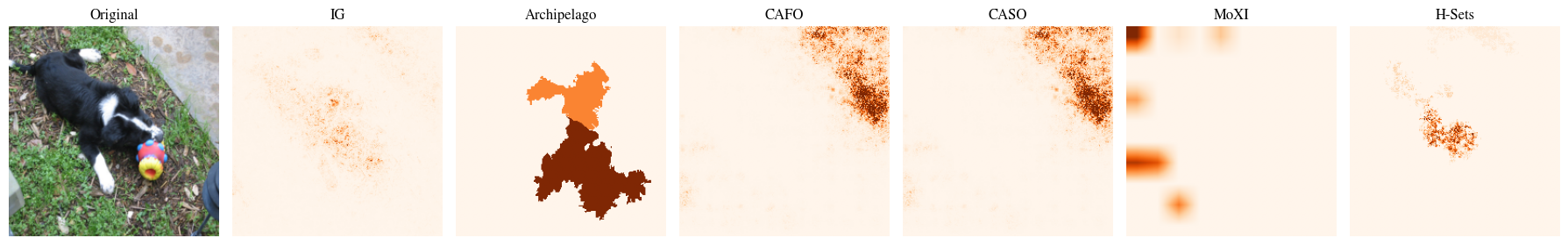}\\[3pt]
    \includegraphics[width=\linewidth]{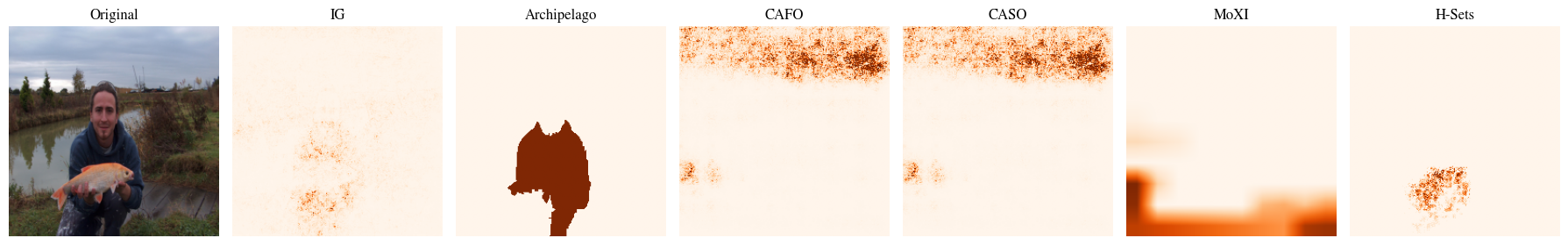}\\[3pt]
    \includegraphics[width=\linewidth]{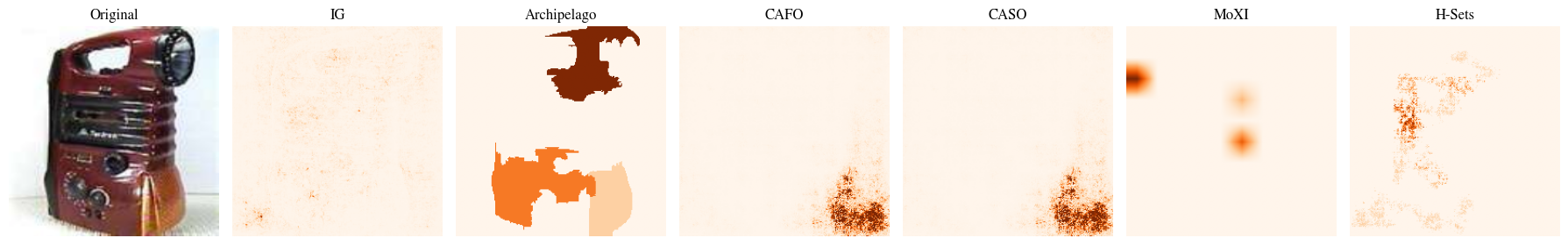}\\[3pt]
    \includegraphics[width=\linewidth]{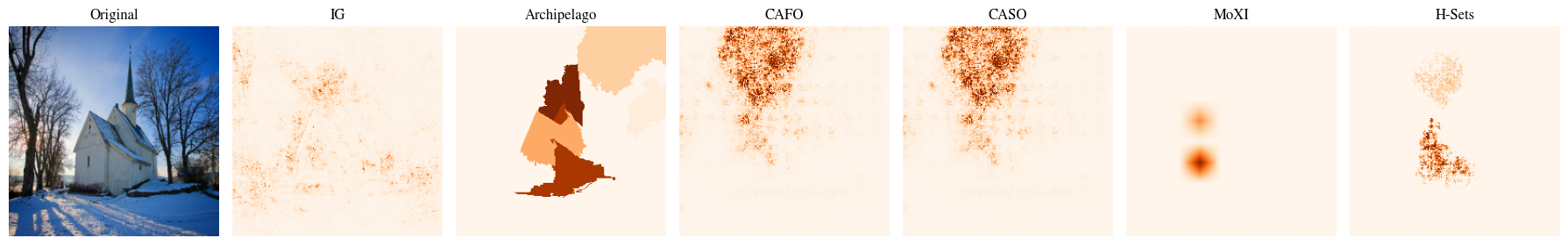}
    \caption{%
        Qualitative comparison of attribution methods .
        Each row shows, from left to right: the original image, and
        saliency maps from Integrated Gradients (IG) \cite{ig}, Archipelago \cite{arch}, Context-aware First Order explanations (CAFO) \cite{caso}, Context-aware Second Order explanations (CASO) \cite{caso}, MOXI~\cite{moxi} and our proposed method, \textbf{H-Sets}. Darker red indicates higher attribution.
    }
    \label{fig:appendix_b}
\end{figure}
 
\clearpage

\section{Sensitivity to Non-Semantic Features}
\label{appendix:decoy_analysis}

A potential concern regarding the use of Segment Anything (SAM)~\cite{sam} as a spatial prior is whether it ``forces" the explanation to be semantic, thereby masking a classifier's reliance on non-semantic shortcuts or spurious correlations. 

To investigate this, we evaluate H-Sets on the DecoyMNIST~\cite{erion2021improving}. In this setup, a digit classifier is intentionally trained on data where a fixed, non-semantic ``decoy'' patch is correlated with the class label. A faithful attribution method must be able to ignore the semantic digit and localize the spurious patch if that is what the model is truly utilizing for its prediction.

As shown in Figure~\ref{fig:decoy}, H-Sets correctly identifies these non-semantic regions as highly salient. This is because the core of our interaction discovery is \textit{curvature-driven} via the input Hessian. While SAM provides a grouping prior to organize these interactions into coherent sets, the actual attribution scores—calculated using Harsanyi dividends and IDG-Vis—are derived strictly from second-order derivatives. 

\begin{figure}[h]
    \centering
    \includegraphics[width=0.5\linewidth]{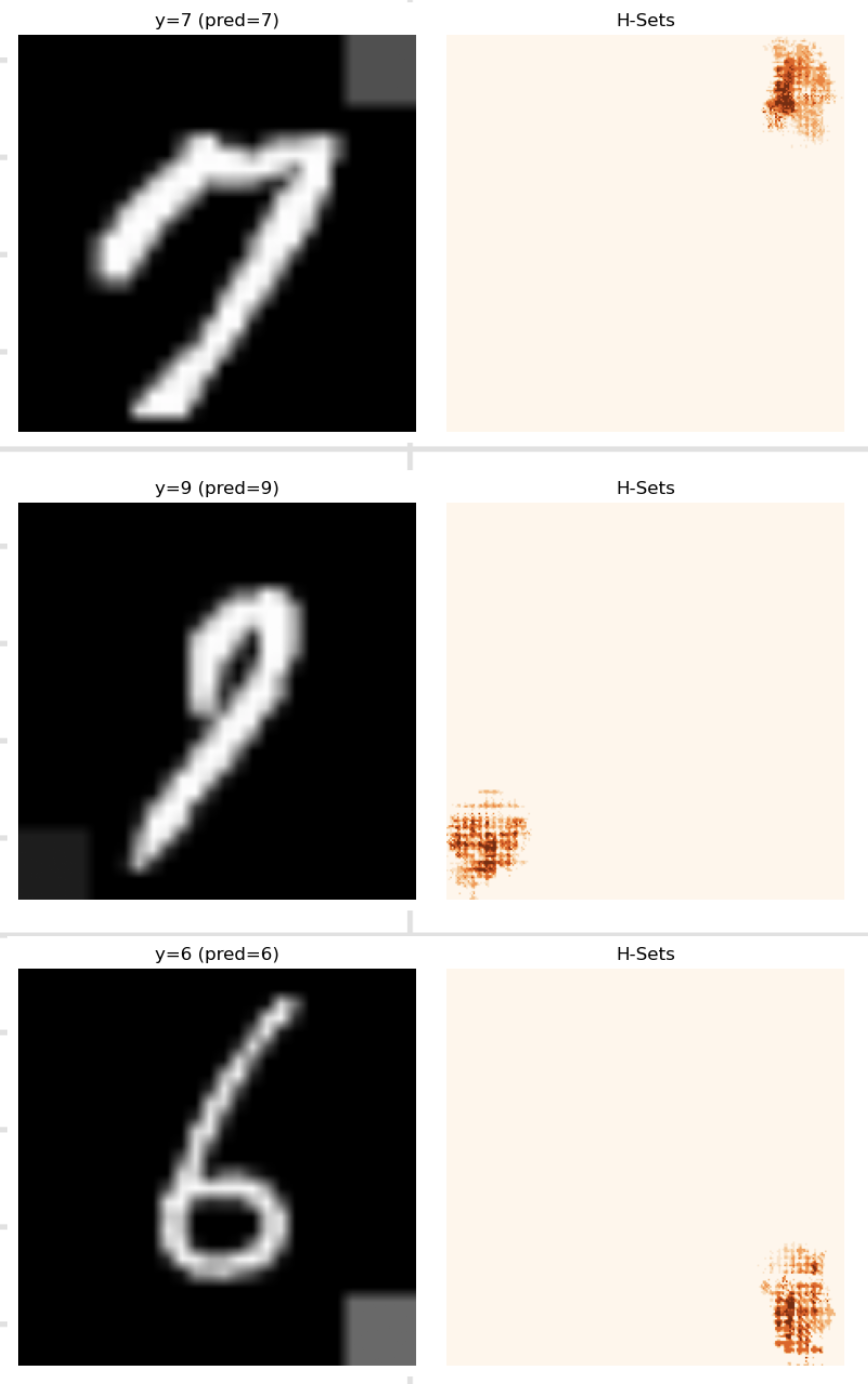}
    \caption{H-Sets attribution on DecoyMNIST. Despite the presence of a strong semantic prior (the digit), H-Sets successfully localizes the spurious decoy patch, demonstrating that the method captures the model's true reliance on non-semantic shortcuts.}
    \label{fig:decoy}
\end{figure}


\section{Algorithms}\label{appendix:algorithms}
Algorithms for H-Sets and IDG-Vis are provided in \cref{alg:hessian} and \cref{alg:dividends}.

\begin{algorithm}[h]
\caption{H-Sets Algorithm}\label{alg:hessian}
\scriptsize
    \begin{algorithmic}
        \Function{GetSet}{example $\mathbf{x}$, model $f$, gradient $\nabla f(\mathbf{x})$, index $j$, threshold $\mu$, max elements $\nu$, interaction set $\mathcal{I}$}
        \State $\mathbf{H}_j \gets \frac{\partial}{ \partial x_i} (\nabla f(\mathbf{x}))_j$  \Comment{\textit{Create Hessian}}
        \State $\mathbf{H}_j' \gets \{i \in \{1, 2, ..., d\} | \mathbf{H}_j[i] > \mu\}$  \Comment{\textit{Consider interactions above threshold}}
        \State next $\gets \{\}$
        \ForAll{$n \in \mathbf{H}_j'$}
        \If{$n = 0$}
            \State continue
        \ElsIf{$n \notin \mathcal{I}$ and $|\mathcal{I}| < \nu$} 
            \State $\mathcal{I} \gets \mathcal{I} \cup n$ \Comment{\textit{Add interaction to current feature interaction set}}
            \State next $\gets $ next $ \cup \, n$
        \EndIf
        \EndFor
        \ForAll{$m \in $ next}
        \State \Call  {GetSet}{$\mathbf{x}$, $f$, $\nabla f(\mathbf{x})$, $\text{argmin}_{j} \{ \text{next}[j] = m \}$, $\mu$, $\nu$, $\mathcal{I}$} \Comment{\textit{Find more interactions}}
        \EndFor
        \State \Return $\mathcal{I}$
        \EndFunction
        \\
        \Function{GenerateSets}{example $\mathbf{x}$, model $f$, threshold $\mu$, max elements $\nu$, number of interaction sets $k$}
        \State masks $\gets$ SAM($\mathbf{x}$); 
        \State indexes $\gets \text{sort(IntegratedGradients($\mathbf{x}, f$))}$ \Comment{\textit{Get Starting Features}}
        \State $\mathcal{S} \gets \{\}$
        \For {index $ \in $ indexes}
        \If{index $\in \{i | i \in m, \forall m \in \text{masks}\}$}
        \State $\mathcal{I} \gets$ \Call  {GetSet}{$(\mathbf{x}, f, \nabla f(\mathbf{x}), \text{index}, \mu, \nu, \mathcal{I})$}
        \State $\mathcal{S} \gets \mathcal{S} \cup \mathcal{I}$ \Comment{\textit{Add Interaction Set into Total Set}}
        \State masks $\gets \{ m \in \text{masks} | \text{index} \notin m \}$
        \EndIf
        \If{$|S| = k$} \\
        \quad\quad\quad\quad\quad\textbf{break}
        \EndIf
        \EndFor
        \State \Return $\mathcal{S}$
        \EndFunction
    \end{algorithmic}
\end{algorithm}

\begin{algorithm}[ht]
\caption{Integrated Directional Gradients for Vision (IDG-Vis)}\label{alg:dividends}
\scriptsize
    \begin{algorithmic}
        \Function{AttributeSets}{example $\mathbf{x}$, model $f$, interaction set $\mathcal{I}$}
        \State $\mathbf{x}' \gets \mathbf{0}^{H \times W \times C}$ \Comment{\textit{Create baseline}}
        \State $a \gets 0$ 
        \ForAll{$T \sim \mathcal{ U}(\mathcal{P}(\mathcal{I}))$}
            \State $\mathbf{a} \gets x_i \in T \ ?\ (\mathbf{x} - \mathbf{x}') : 0$
            \State $\hat{\mathbf{a}} \gets \frac{\mathbf{a}}{||\mathbf{a}||}$ \Comment{\textit{Create directional vector}}
            \State $A \gets 0$ 
            \For{$k \gets 0$ \textbf{to} $m$}
                \State $\nabla_T f(\mathbf{x}) \gets \nabla f(\mathbf{x}' + \frac{k}{m}(\mathbf{x} - \mathbf{x}')) \cdot \hat{\mathbf{a}}$ \Comment{\textit{Compute Directional Gradient}}
                \State $A \gets A + \nabla_T f(\mathbf{x}) $ \Comment{\textit{Create Integrated Directional Gradient for Vision}}
            \EndFor
            \State $A \gets \frac{1}{m+1} \cdot (A)$
            \State $a \gets a + A$
        \EndFor
        \EndFunction
    \end{algorithmic}
\end{algorithm}

\section{Proof for axioms}\label{appendix:proof}

In this section, we show that our proposed attribution method satisfies all the axioms listed in \cref{sec:axioms}.

\begin{proposition}
$a(\mathcal{I})$ satisfies \textnormal{Axioms} 1-4
\end{proposition}
Since the TU-game is always positive, it must be true that non-negativity, normality, monotonicity, and superadditivity are satisfied \cite{dehez2017harsanyi}. 

\begin{lemma}[Approximate Completeness]
Let $\mathcal{S}$ be the set of feature interaction sets detected by H-Sets for an input $\mathbf{x}$ and baseline $\mathbf{x}'$. Then the total interaction attribution satisfies:
\[
\sum_{\mathcal{I} \in \mathcal{S}} a(\mathcal{I}) \leq f(\mathbf{x}) - f(\mathbf{x}')
\]
with equality if and only if the union of subsets $T \sim \mathcal{U}(\mathcal{P}(\mathcal{I}))$ for all $\mathcal{I} \in \mathcal{S}$ covers the full support of the input difference vector $\mathbf{x} - \mathbf{x}'$.
\end{lemma}

\begin{proof}
We define the linear interpolation path $\mathbf{x}(\alpha) = \mathbf{x}' + \alpha(\mathbf{x} - \mathbf{x}')$, for $\alpha \in [0, 1]$.

By definition, the directional gradient for interaction set $\mathcal{I}$ is:
\[
\nabla_{\mathcal{I}} f(\mathbf{x}(\alpha)) = \left| \nabla f(\mathbf{x}(\alpha)) \cdot \hat{\mathbf{a}}_{\mathcal{I}} \right|
\]
and the integrated directional gradient is:
\[
a(\mathcal{I}) \approx \sum_{T \sim \mathcal{U}(\mathcal{P}(\mathcal{I}))} \int_0^1 \left| \nabla f(\mathbf{x}(\alpha)) \cdot \hat{\mathbf{a}}_T \right| \, d\alpha
\]

Summing over all sets $\mathcal{I} \in \mathcal{S}$ yields (“integral of sums” from “sum of integrals”):
\[
\sum_{\mathcal{I} \in \mathcal{S}} a(\mathcal{I}) \approx \int_0^1 \sum_{\mathcal{I} \in \mathcal{S}} \sum_{T \sim \mathcal{U}(\mathcal{P}(\mathcal{I}))} \left| \nabla f(\mathbf{x}(\alpha)) \cdot \hat{\mathbf{a}}_T \right| \, d\alpha
\]

The total directional gradient, from Integrated Gradients, is given by:
\[
\int_0^1 \nabla f(\mathbf{x}(\alpha)) \cdot \frac{d\mathbf{x}(\alpha)}{d\alpha} \, d\alpha = f(\mathbf{x}) - f(\mathbf{x}')
\]

Our method calculates the total attribution by summing the scores $a(I)$ for a detected subset of interaction sets $S$. However, we calculate a sum of attributions for only a select group of interactions. Hence, $a(I)$ under-approximates the total directional gradient, unless the set of direction vectors $\hat{\mathbf{a}}_T$ spans the direction of $(\mathbf{x} - \mathbf{x}')$. Since $\mathcal{S}$ only includes a subset of all possible interaction sets, we obtain:
\[
\sum_{\mathcal{I} \in \mathcal{S}} a(\mathcal{I}) \le f(\mathbf{x}) - f(\mathbf{x}')
\]
Equality holds when the sampled subset directions fully represent the input difference vector.
\end{proof}

\begin{lemma}
    If an input $\mathbf{x}$ and a baseline $\mathbf{x'}$ are equal everywhere except $\mathbf{x}_\mathcal{I} \neq \mathbf{x'}_\mathcal{I}$ and if $f(\mathbf{x}) \neq f(\mathbf{x'})$, then $a(\mathcal{I}) \neq 0$. In other words, $a(\mathcal{I})$ satisfies Axiom 6.
\end{lemma}

\begin{proof}

We have two inputs, \( \mathbf{x} \) and \( \mathbf{x'} \), such that \( \mathbf{x} \) and \( \mathbf{x'} \) differ only on the subset of features \( \mathcal{I} \): $\mathbf{x}_\mathcal{I} \neq \mathbf{x'}_\mathcal{I}$, while for any \( j \notin \mathcal{I} \), \( \mathbf{x}_j = \mathbf{x'}_j \). Additionally, it is given that the model’s outputs at these points are different: $f(\mathbf{x}) \neq f(\mathbf{x'})$. 

The attribution \( a(\mathcal{I}) \) for a set of features \( \mathcal{I} \) is defined by the Integrated Directional Gradients (IDG) approach:
     \[
     a(\mathcal{I}) = \int_{\alpha=0}^1 \nabla_\mathcal{I} f(\mathbf{x}' + \alpha(\mathbf{x} - \mathbf{x'})) \, d\alpha
     \]

This integral measures the cumulative effect of changing the features in \( \mathcal{I} \) from \( \mathbf{x'} \) to \( \mathbf{x} \) on the model’s output \( f \), along a linear path parameterized by \( \alpha \) from 0 to 1. Since \( f(\mathbf{x}) \neq f(\mathbf{x'}) \), we know that there is a non-zero change in the model’s output as we move along the path from \( \mathbf{x'} \) to \( \mathbf{x} \). Because \( \mathbf{x} \) and \( \mathbf{x'} \) differ only on the features in \( \mathcal{I} \), any change in \( f \) along this path is attributed solely to the variations in the features within \( \mathcal{I} \). Hence, this integral must capture a non-zero contribution from \( \mathcal{I} \), resulting in \( a(\mathcal{I}) \neq 0 \).
\end{proof}



\begin{lemma}
    Two neural networks $f(\cdot)$ and $f'(\cdot)$, with corresponding value functions $a'$ and $a''$, are functionally equivalent if $f'(\mathbf{x}) = f''(\mathbf{x})$ for all $\mathbf{x}$. Then, $a'(\mathcal{I}) = a''(\mathcal{I})$ for sets $\mathcal{I} \in \mathcal{S}$. $a(\mathcal{I})$ satisfies Axiom 7.
\end{lemma}

\begin{proof}
If two neural networks \( f' \) and \( f'' \) are functionally equivalent, meaning \( f'(\mathbf{x}) = f''(\mathbf{x}) \) for all inputs \( \mathbf{x} \), their attributions for any feature set \( \mathcal{I} \in \mathcal{S} \) will also be identical, provided they use the same path for gradient computation. IDG computes attribution by integrating over a defined path from a baseline \( \mathbf{x'} \) to the input \( \mathbf{x} \), so having both networks follow this same path ensures that their gradients are identical at each point along it. Consequently, the attributions \( a'(\mathcal{I}) \) and \( a''(\mathcal{I}) \) will match for all feature sets \( \mathcal{I} \in \mathcal{S} \), since the path-integrated contributions are equal. Thus, \( a'(\mathcal{I}) = a''(\mathcal{I}) \) for all \( \mathcal{I} \in \mathcal{S} \).
\end{proof}

\begin{lemma}
    Given two neural networks $f'(\cdot)$ and $f''(\cdot)$, and $f(\mathbf{x}) = c \cdot f'(\mathbf{x}) + d \cdot f''(\mathbf{x})$, then $a(\mathcal{I}) =c \cdot a'(\mathcal{I}) + d \cdot a''(\mathcal{I})$. $a(\mathcal{I})$ satisfies Axiom 8.
\end{lemma}

\begin{proof}
Attribution \( a(\mathcal{I}) \) for a feature set \( \mathcal{I} \) is computed using the IDG, which involves integrating the gradient along a path from a baseline \( \mathbf{x}' \) to the input \( \mathbf{x} \):
     \[
     a(\mathcal{I}) = \int_{\alpha=0}^1 \nabla_\mathcal{I} f(\mathbf{x}' + \alpha(\mathbf{x} - \mathbf{x'})) \, d\alpha.
     \]

Similarly, for \( f' \) and \( f'' \), we have:
     \[
     a'(\mathcal{I}) = \int_{\alpha=0}^1 \nabla_\mathcal{I} f'(\mathbf{x}' + \alpha(\mathbf{x} - \mathbf{x'})) \, d\alpha
     \]
     and
     \[
     a''(\mathcal{I}) = \int_{\alpha=0}^1 \nabla_\mathcal{I} f''(\mathbf{x}' + \alpha(\mathbf{x} - \mathbf{x'})) \, d\alpha.
     \]

Since \( f(\mathbf{x}) = c \cdot f'(\mathbf{x}) + d \cdot f''(\mathbf{x}) \), the gradient of \( f \) with respect to the features in \( \mathcal{I} \) is:
     \[
     \nabla_\mathcal{I} f(\mathbf{x}) = c \cdot \nabla_\mathcal{I} f'(\mathbf{x}) + d \cdot \nabla_\mathcal{I} f''(\mathbf{x}).
     \]

This linearity holds for any point along the path \( \mathbf{x}(\alpha) = \mathbf{x}' + \alpha(\mathbf{x} - \mathbf{x'}) \), so we have:
     \[
     \nabla_\mathcal{I} f(\mathbf{x}(\alpha)) = c \cdot \nabla_\mathcal{I} f'(\mathbf{x}(\alpha)) + d \cdot \nabla_\mathcal{I} f''(\mathbf{x}(\alpha)).
     \]

We can now substitute this expression into the integral that defines \( a(\mathcal{I}) \):
     \[
     a(\mathcal{I}) = \int_{\alpha=0}^1 \nabla_\mathcal{I} f(\mathbf{x}(\alpha)) \, d\alpha
     \]
     \[
     = \int_{\alpha=0}^1 \left( c \cdot \nabla_\mathcal{I} f'(\mathbf{x}(\alpha)) + d \cdot \nabla_\mathcal{I} f''(\mathbf{x}(\alpha)) \right) \, d\alpha.
     \]

By the linearity of integration, we can separate the terms inside the integral:
     \[
     a(\mathcal{I}) = c \cdot \int_{\alpha=0}^1 \nabla_\mathcal{I} f'(\mathbf{x}(\alpha)) \, d\alpha + d \cdot \int_{\alpha=0}^1 \nabla_\mathcal{I} f''(\mathbf{x}(\alpha)) \, d\alpha.
     \]

This simplifies to:
     \[
     a(\mathcal{I}) = c \cdot a'(\mathcal{I}) + d \cdot a''(\mathcal{I}).
     \]





\end{proof}

\begin{lemma}
    Given two symmetric feature interaction sets $\mathcal{I}_1$ and $\mathcal{I}_2$, the value function for neural network $f$ will be $a(\mathcal{I}_1) = a(\mathcal{I}_2)$. In other words, $a(\mathcal{I})$ satisfies Axiom 9.
\end{lemma}

\begin{proof}
  Since \( \mathcal{I}_1 \) and \( \mathcal{I}_2 \) are symmetric, swapping the features in \( \mathcal{I}_1 \) with those in \( \mathcal{I}_2 \) does not change the output of \( f \). This symmetry implies that the gradients with respect to the features in \( \mathcal{I}_1 \) and \( \mathcal{I}_2 \) are equal:
    \[
    \nabla_{\mathcal{I}_1} f(\mathbf{x}) = \nabla_{\mathcal{I}_2} f(\mathbf{x}).
    \]

 Furthermore, because there must be a subset \( T_1 \subseteq \mathcal{I}_1 \) that has a corresponding symmetric subset \( T_2 \subseteq \mathcal{I}_2 \) (with \( |T_1| = |T_2| \)), then we also have:
    \[
    \nabla_{T_1} f(\mathbf{x}) = \nabla_{T_2} f(\mathbf{x}).
    \]

    Therefore, we can write the sum of the attributions as follows:
    \[
    \sum_{T_1 \sim \mathcal{U}(\mathcal{P}(\mathcal{I}_1))} |\nabla_{T_1} f(\mathbf{x})| = \sum_{T_2 \sim \mathcal{U}(\mathcal{P}(\mathcal{I}_2))} |\nabla_{T_2} f(\mathbf{x})|.
    \]

    Since the attribution method (e.g., Integrated Directional Gradients) involves integrating the gradients over a path from the baseline \( \mathbf{x'} \) to \( \mathbf{x} \), we can express the attribution for each interaction set as:

     \begin{equation*}
    \begin{split}
        \sum_{T_1 \sim \mathcal{U}(\mathcal{P}(\mathcal{I}_1))} \int_{\alpha = 0}^1 |\nabla_{T_1} &f(\mathbf{x}(\alpha)| =  \\ &\sum_{T_2 \sim \mathcal{U}(\mathcal{P}(\mathcal{I}_2))} \int_{\alpha = 0}^1 |\nabla_{T_2} f(\mathbf{x}(\alpha))|
    \end{split}
    \end{equation*}

Therefore, 

 \begin{equation*}
        a(\mathcal{I}_1) = a(\mathcal{I}_2)
    \end{equation*}
    
\end{proof}

\section{Smoothing RELU}\label{appendix:smoothingrelu}
Instead of using the SoftPlus $s(z)$ like previous works \cite{ih}, we use the following approximation of the ReLU $h(z)$ from Zhang \etal \cite{fire} as it provided less noisy explanations. \cref{fig:grad-hess} shows the plots for gradient and hessian for the functions with $\tau = 0.001$.

\begin{equation}
    h(z) =
    \begin{cases}
    \left( z + \sqrt{z^2 + \tau} \right)' = 1 + \frac{z}{\sqrt{z^2 + \tau}} & (z < 0) \\
    \left( \sqrt{z^2 + \tau} \right)' = \frac{z}{\sqrt{z^2 + \tau}} & (z \geq 0)
    \end{cases}
\end{equation}

\begin{figure}[h]
    \centering
    \includegraphics[width=0.60\linewidth]{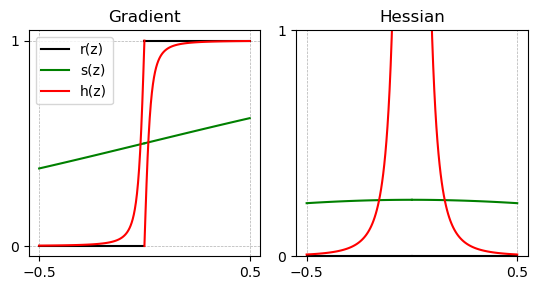}
    \caption{Gradient and Hessian comparison between the ReLU activation function $r(z)$, SoftPlus activation function $s(z)$, and Zhang \etal \cite{fire} $h(z)$ }
    \label{fig:grad-hess}
\end{figure}

\section{Higher-order feature interaction}\label{appendix:higherorder}

\begin{figure}[h]
    \centering
    \includegraphics[width=0.45\linewidth]{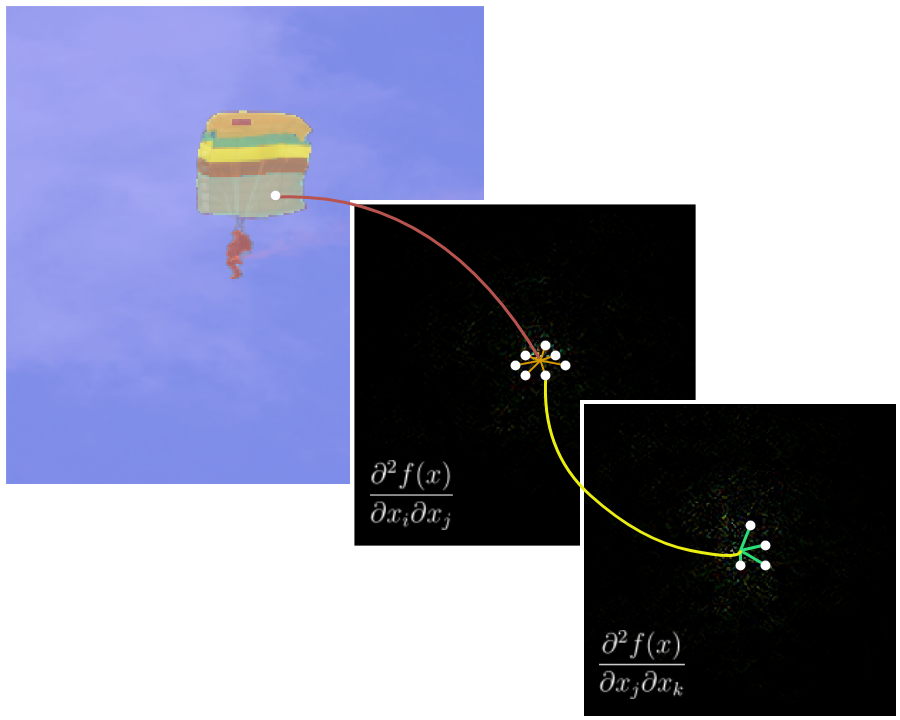}
    \caption{Diagram showing H-Sets algorithm. First, the image is segmented using SAM. Then, the highest attributed feature by Integrated Gradients in a mask is used as a starting feature to kickstart the algorithm. The Hessian matrix is then taken using this feature and the highest attributed features from the Hessian are appended to the feature interaction set. If more features can be added to the set, then the Hessian matrix of the highest attributed features from the previous Hessian is taken and the algorithm repeats recursively.}
    \label{fig:pairwise}
\end{figure}

As discussed in Section~\ref{section:higherorderfeature}, we propose to fix the starting feature $x_i$ as the highest attributed feature by Integrated Gradients~\cite{ig} in segmentation masks produced by Segment Anything Model (SAM) \cite{sam}. The overview of the H-Sets algorithm is shown in Figure~\ref{fig:pairwise}.

\section{Baselines}\label{appendix:baseline}

We evaluate our method against non-interaction-based attribution methods, Integrated Gradients~\cite{ig}, and interaction-based attribution methods: CAFO, CASO~\cite {caso}, Archipelago~\cite {arch} and MoXi~\cite{moxi}.

\subsection{Integrated Gradient}
Integrated Gradients (IG) assigns importance scores to each input feature by integrating the model’s gradients along a straight-line path from a baseline input $\mathbf{x}'$ (e.g., a black image or zero vector) to the actual input $\mathbf{x}$. This approach ensures that the attributions satisfy desirable properties such as sensitivity and implementation invariance.

Formally, given a model $f: \mathbb{R}^d \to \mathbb{R}$ and an input $\mathbf{x} \in \mathbb{R}^d$, the Integrated Gradient of the $i$-th input feature is defined as:

$$
\text{IG}_i(\mathbf{x}) = (x_i - x'_i) \int_{\alpha=0}^1 \frac{\partial f(\mathbf{x}' + \alpha (\mathbf{x} - \mathbf{x}'))}{\partial x_i} \, d\alpha
$$

where $\mathbf{x}'$ is the baseline input, and $\alpha \in [0,1]$ is the interpolation parameter. In practice, the integral is approximated using a Riemann sum over $m$ steps:

$$
\text{IG}_i(\mathbf{x}) \approx (x_i - x'_i) \cdot \frac{1}{m} \sum_{k=1}^m \frac{\partial f\left(\mathbf{x}' + \frac{k}{m}(\mathbf{x} - \mathbf{x}')\right)}{\partial x_i}
$$

Integrated Gradients (IG) is one of the most widely used feature attribution methods. We use the implementation of IG available in the Captum library~\cite{kokhlikyan2020captum}.

\subsection{Context-Aware First-Order and Second-Order Attributions}

To provide structured and faithful feature attributions, the Context-Aware First-Order (CAFO) and Context-Aware Second-Order (CASO) methods formulate feature importance as an optimization problem that explicitly incorporates group-feature perturbations under sparsity and smoothness constraints \cite{singla2019understanding}. Given a trained model \( f_{\theta^*} \), an input-label pair \( (x, y) \), and the loss function \( \ell(f_{\theta^*}(x), y) \), both CAFO and CASO aim to find a perturbation \( \Delta = \tilde{x} - x \) that maximizes the change in loss while penalizing large and non-sparse perturbations. This formulation captures the idea that important feature subsets are those whose perturbation causes significant changes in model behavior.

The \textbf{CAFO} method linearizes the loss function around the input and solves the following optimization:
\begin{equation}
\tilde{I}^{\text{CAFO}}_{\lambda_1, \lambda_2}(x, y) = \max_{\Delta} \nabla_x \ell(f_{\theta^*}(x), y)^\top \Delta - \lambda_1 \|\Delta\|_1 - \lambda_2 \|\Delta\|_2^2
\end{equation}
where \( \lambda_1 \) and \( \lambda_2 \) are regularization parameters controlling the sparsity and magnitude of the perturbation.

To better approximate the true model behavior, the \textbf{CASO} method extends CAFO by using a second-order Taylor expansion of the loss function:

\begin{equation}
\begin{split}
\tilde{I}^{\text{CASO}}_{\lambda_1, \lambda_2}(x, y) = \max_{\Delta} \nabla_x \ell(f_{\theta^*}(x), y)^\top \Delta + \\ \frac{1}{2} \Delta^\top H_x \Delta \\ - \lambda_1 \|\Delta\|_1 - \lambda_2 \|\Delta\|_2^2
\end{split}
\end{equation}
where \( H_x \) is the Hessian of the loss with respect to the input features. The inclusion of the curvature term \( \Delta^\top H_x \Delta \) enables CASO to capture model curvature.

We use the author's official implementation available in Github \cite{caso2021github}. 

\subsection{Archipelago}
{Archipelago} is a model-agnostic framework for detecting and attributing feature interactions \cite{arch}. It consists of two main components: {ArchAttribute}, an interaction-aware attribution method, and {ArchDetect}, a scalable interaction detection algorithm. Together, they generate axiomatic explanations by isolating groups of features—referred to as "feature islands"—whose combined contributions drive model predictions.

Given a black-box model \( f \), an input instance \( x^\star \), and a baseline \( x' \), ArchAttribute assigns an attribution score \( \phi(I) \) to each feature set \( I \subseteq [d] \), defined as:
\begin{equation}
\phi(I) = f(x^\star_I + x'_{\setminus I}) - f(x')
\end{equation}
This formulation captures the isolated contribution of the feature group \( I \) by embedding it in a neutral baseline context \( x' \). ArchAttribute is designed to be additive and satisfies several generalized axioms. To efficiently identify relevant interaction sets \( \mathcal{S} = \{I_1, I_2, \dots\} \), Archipelago also introduces {ArchDetect}. Please see the paper \cite{arch} for more details.

We use the official implementation of Archipelago provided by the authors on GitHub~\cite{archgithub}. For a fair comparison, we fix the number of feature interaction sets to 5 in Archipelago. We use the same number of interaction sets in our method.

\subsection{MoXi}

\textbf{MoXI} ({Model eXplanation by Interactions}) is a recent method that efficiently identifies and attributes groups of pixels using Shapley values and their interactions~\cite{moxi}. The framework provides two complementary approaches: \textit{pixel insertion}, which measures the confidence gain when a pixel is revealed, and \textit{pixel deletion}, which measures the confidence drop when a pixel is masked. In this paper, we employ the pixel insertion variant.

Because game-theoretic quantities such as Shapley values and interactions are computationally expensive, MoXI introduces a super-pixel representation and a greedy selection strategy. Let $k$ denote the number of pixels to be added to the interaction set $\mathcal{I}$, $N$ the total number of pixels, and $f$ the model. According to the pixel insertion algorithm, to add a pixel $b_k$ to $\mathcal{I}$, the method selects
\[
b_k \leftarrow \arg\max_{b \in N \setminus \mathcal{I}} f(\mathcal{I} \cup \{b\}),
\]
repeating this process $k$ times to construct the interaction set. 

We use the original implementation of MoXI provided in their official GitHub repository \cite{moxigithub}.

\section{Evaluation Metrics}\label{appendix:metrics}
Saliency maps only provide a qualitative evaluation of an explanation method. However, such evaluation does not provide an accurate comparison because of associated human bias and lack of ``ground-truth" data. In this work, we focus on two important evaluation metrics: sparsity and faithfulness.

\subsection{Sparsity}
Sparsity evaluates the comprehensibility of explanations. We evaluate the sparsity of the attribution vector \(\phi(\mathbf{x})\) by calculating its Gini index, as implemented by Chalasani \etal \cite{chalasani2020concise}. For an attribution vector \(\phi(\mathbf{x}) \in \mathbb{R}^d\), we first sort the absolute values in non-decreasing order and then compute the Gini index using Eqn. \ref{eqn:gini}.

\begin{equation}\label{eqn:gini}
    G(\phi(\mathbf{x})) = 1 - 2 \sum_{k=1}^d \frac{\phi(\mathbf{x})_{(k)}}{||\phi(\mathbf{x})||_1} \frac{d-k+0.5}{d}
\end{equation}

This formula calculates a weighted sum of fractions, where each fraction represents the contribution of the \(k\)-th largest element to the overall sparsity. Larger elements receive higher weights, emphasizing their impact on sparsity. The Gini index ranges from \([0,1]\): a value of 1 indicates perfect sparsity, where only one element in \(\phi(\mathbf{x})\) is greater than zero, while a value of 0 indicates no sparsity, meaning all vector elements are equal to some positive value.

\subsection{Faithfulness}

Faithfulness measures whether the explanation truly reflects the underlying model behavior. We use the ROAD evaluation by Rong \etal \cite{rong2022consistent} for faithfulness. ROAD evaluates model accuracy on a test set during an iterative process where the top \(k\) most important pixels are removed at each step. Pixel removal is performed using a noisy linear imputation technique to minimize the creation of out-of-distribution samples.

For our experiments, we use the MoRF (Most Relevant First) strategy from the ROAD implementation available in Quantus \cite{hedstrom2023quantus}. However, both the MoRF and LeRF (Least Significant Removal First) strategy provides the same results as discussed in the original paper. 

In MoRF, given a model \(F\) and an input, an attribution method assigns importance values to each feature. The features are then ranked in descending order of importance. At each step, the top \(k = 5\) most important features are removed, and model accuracy is assessed. A sharper accuracy drop indicates a more effective explanation.

We opted for ROAD over Insertion/Deletion \cite{petsiukrise} and ROAR \cite{hooker2019benchmark} because Insertion/Deletion introduces artifacts, leading to distribution shifts in perturbed inputs, and ROAR requires costly model retraining.

\paragraph{$\text{ROAD}_\text{AOPC}$ score.} In addition, we quantify the ROAD plot using area over the perturbation curve, computed as $\text{ROAD}_\text{AOPC} = \frac{1}{L + 1} \sum_{k=1}^{L} \left\langle f(x^{(0)}) - f(x^{(k)}) \right\rangle$ where, \( L \) represents the number of feature removal steps, and \( f(x) \) is the classifier's output for the originally predicted class given the input \( x \). The term \( x^{(0)} \) corresponds to the unperturbed input image, while \( x^{(k)} \) represents the image after \( k \) perturbation steps. Higher $\text{ROAD}_\text{AOPC}$ score represents a more faithful method.

\section{Additional experiments}\label{appendix:additionalexperiments}

\subsection{Non-interaction baselines}\label{appendix:morebaseline}
In addition to the baseline methods discussed in \cref{appendix:baseline}, we evaluate H-Sets against DeepLift~\cite{shrikumar2017learning} and LRP~\cite{lrp}. 

DeepLIFT~\cite{shrikumar2017learning} is explains a model's prediction by decomposing the ``difference-from-reference" of the output into contributions from each input feature. It operates by comparing the activations of the actual input $x$ against a user-defined ``reference" input $x_0$ (e.g., a neutral background or blurred image). By defining the difference $\Delta t = t - t_0$ for a target neuron $t$, the method assigns contribution scores $C_{\Delta x_i \Delta t}$ that satisfy the summation-to-delta property: $\sum C_{\Delta x_i \Delta t} = \Delta t$. 

Layer-wise Relevance Propagation (LRP)~\cite{lrp} redistributes the prediction score (relevance) from the output layer back to the input pixels, following the fundamental conservation principle: the total relevance $R$ must be preserved across each layer of the network ($\sum R_{i} = \sum R_{j}$). In this framework, each neuron distributes its relevance to its predecessors based on their relative contribution to its activation, using specific decomposition rules to handle different weight distributions. 

\cref{tab:sparsityadditional} and \cref{tab:faithfulnessadditional} demonstrate the sparsity and faithfulness scores. Compared with the main results in \cref{tab:sparsity,tab:faithfulnesseval}, these additional non-interaction baselines further reinforce the same overall pattern. In terms of sparsity, DeepLift and LRP are occasionally competitive with first-order baselines. However, this sparsity does not translate into consistently strong faithfulness. Across most architectures, their $\mathrm{ROAD}_{\mathrm{AOPC}}$ scores remain below those of \textbf{H-Sets}. In contrast, \textbf{H-Sets} maintains a stronger balance between sparsity and faithfulness across datasets and architectures.

\begin{table}[t]
\centering
\caption{
Sparsity comparison (Gini index; higher is better) of LRP and DeepLift. Values are mean $\pm$ std for 5 runs.}
\label{tab:sparsityadditional}
\resizebox{0.48\textwidth}{!}{
\begin{tabular}{l l c c}
\toprule
\textbf{Dataset} & \textbf{Model} & \textbf{LRP} & \textbf{DeepLift} \\
\midrule

\multirow{4}{*}{\textbf{ImageNet}} 
& VGG       & 0.681{\scriptsize$\pm$0.068} & 0.721{\scriptsize$\pm$0.073} \\
& ResNet    & 0.965{\scriptsize$\pm$0.022} & 0.601{\scriptsize$\pm$0.054} \\
& DenseNet  & 0.627{\scriptsize$\pm$0.078} & 0.611{\scriptsize$\pm$0.077} \\
& MobileNet & 0.584{\scriptsize$\pm$0.057} & 0.567{\scriptsize$\pm$0.053} \\
\midrule

\multirow{4}{*}{\textbf{CUB}} 
& VGG       & 0.752{\scriptsize$\pm$0.135} & 0.703{\scriptsize$\pm$0.082} \\
& ResNet    & 0.983{\scriptsize$\pm$0.022} & 0.586{\scriptsize$\pm$0.061} \\
& DenseNet  & 0.690{\scriptsize$\pm$0.170} & 0.556{\scriptsize$\pm$0.061} \\
& MobileNet & 0.600{\scriptsize$\pm$0.059} & 0.563{\scriptsize$\pm$0.062} \\
\bottomrule
\end{tabular}
}
\end{table}

\begin{table}[t]
\centering
\caption{
Faithfulness comparison ($\mathrm{ROAD}_{\mathrm{AOPC}}$; higher is better) of LRP and DeepLift. Values are mean $\pm$ std for 5 runs.}
\label{tab:faithfulnessadditional}
\resizebox{0.48\textwidth}{!}{
\begin{tabular}{l l c c}
\toprule
\textbf{Dataset} & \textbf{Model} & \textbf{LRP} & \textbf{DeepLift} \\
\midrule

\multirow{4}{*}{\textbf{ImageNet}} 
& VGG       & 0.357{\scriptsize$\pm$0.023} & 0.431{\scriptsize$\pm$0.073} \\
& ResNet    & 0.231{\scriptsize$\pm$0.042} & 0.224{\scriptsize$\pm$0.051} \\
& DenseNet  & 0.265{\scriptsize$\pm$0.043} & 0.425{\scriptsize$\pm$0.027} \\
& MobileNet & 0.271{\scriptsize$\pm$0.034} & 0.348{\scriptsize$\pm$0.056} \\
\midrule

\multirow{4}{*}{\textbf{CUB}} 
& VGG       & 0.565{\scriptsize$\pm$0.073} & 0.566{\scriptsize$\pm$0.085} \\
& ResNet    & 0.375{\scriptsize$\pm$0.021} & 0.437{\scriptsize$\pm$0.054} \\
& DenseNet  & 0.385{\scriptsize$\pm$0.035} & 0.584{\scriptsize$\pm$0.076} \\
& MobileNet & 0.529{\scriptsize$\pm$0.094} & 0.448{\scriptsize$\pm$0.073} \\
\bottomrule
\end{tabular}
}
\end{table}




\subsection{Additional faithfulness metric}\label{appendix:moremetrics}
In addition to ROAD~\cite{rong2022consistent}, we evaluate the faithfulness using faithfulness correlation~\cite{bhatt2020evaluating}. It operates by iteratively selecting a random subset of features ($|S|$), replacing them with baseline values, and then calculating the Pearson’s correlation coefficient between the average attribution of those features and the resulting change in the model's predicted logits. By averaging these correlations over multiple runs and test samples, the metric produces a correlation score, where higher values indicate a more ``faithful" explanation. 

\cref{tab:faithfulnesscorrelation} demonstrates that H-Sets consistently outperforms all other methods across datasets and networks on this metric. 

\begin{table}[t]
\centering
\caption{
Faithfulness Correlation (higher is better) of our method (H-Sets) versus Integrated Gradients (IG)~\cite{ig}, context-aware first-order (CAFO), and second-order explanations (CASO)~\cite{caso}, Archipelago (Arch)~\cite{arch}, and MOXI~\cite{moxi}. Values are mean $\pm$ std for 5 runs.}
\label{tab:faithfulnesscorrelation}
\resizebox{0.48\textwidth}{!}{
\begin{tabular}{l l c c c c c c}
\toprule
\multicolumn{2}{c}{} &
\multicolumn{6}{c}{\textbf{Explanation Methods}} \\
\cmidrule(l){3-8}
\textbf{Dataset} & \textbf{Model} & IG & Arch & CAFO & CASO & MoXI & \textbf{H-Sets} \\
\midrule

\multirow{4}{*}{\textbf{ImageNet}} 
& VGG        & -0.017{\scriptsize$\pm$0.106} & 0.009{\scriptsize$\pm$0.108} & -0.007{\scriptsize$\pm$0.102} & -0.001{\scriptsize$\pm$0.101} & 0.007{\scriptsize$\pm$0.119} & \textbf{0.009{\scriptsize$\pm$0.122}} \\
& ResNet     & -0.016{\scriptsize$\pm$0.099} & -0.044{\scriptsize$\pm$0.381} & 0.004{\scriptsize$\pm$0.109} & 0.004{\scriptsize$\pm$0.118} & -0.024{\scriptsize$\pm$0.349} & \textbf{0.251{\scriptsize$\pm$0.101}} \\
& DenseNet   & -0.005{\scriptsize$\pm$0.127} & \textbf{0.059{\scriptsize$\pm$0.147}} & -0.037{\scriptsize$\pm$0.121} & -0.032{\scriptsize$\pm$0.119} & 0.012{\scriptsize$\pm$0.126} & 0.034{\scriptsize$\pm$0.135} \\
& MobileNet  & -0.015{\scriptsize$\pm$0.117} & 0.020{\scriptsize$\pm$0.138} & -0.001{\scriptsize$\pm$0.105} & 0.000{\scriptsize$\pm$0.104} & 0.011{\scriptsize$\pm$0.129} & \textbf{0.021{\scriptsize$\pm$0.113}} \\
\midrule

\multirow{4}{*}{\textbf{CUB}} 
& VGG        & -0.015{\scriptsize$\pm$0.119} & 0.033{\scriptsize$\pm$0.317} & 0.009{\scriptsize$\pm$0.355} & 0.008{\scriptsize$\pm$0.111} & -0.009{\scriptsize$\pm$0.136} & \textbf{0.124{\scriptsize$\pm$0.157}} \\
& ResNet     & -0.008{\scriptsize$\pm$0.118} & 0.122{\scriptsize$\pm$0.319} & 0.016{\scriptsize$\pm$0.103} & 0.003{\scriptsize$\pm$0.111} & 0.037{\scriptsize$\pm$0.341} & \textbf{0.123{\scriptsize$\pm$0.426}} \\
& DenseNet   & -0.008{\scriptsize$\pm$0.114} & 0.111{\scriptsize$\pm$0.353} & \textbf{0.144{\scriptsize$\pm$0.340}} & -0.021{\scriptsize$\pm$0.108} & 0.025{\scriptsize$\pm$0.153} & 0.029{\scriptsize$\pm$0.126} \\
& MobileNet  & -0.004{\scriptsize$\pm$0.119} & 0.035{\scriptsize$\pm$0.303} & 0.092{\scriptsize$\pm$0.302} & -0.001{\scriptsize$\pm$0.104} & 0.013{\scriptsize$\pm$0.110} & \textbf{0.128{\scriptsize$\pm$0.119}} \\
\bottomrule
\end{tabular}
}
\end{table}

\section{ROAD plots}\label{appendix:ROADplots}
\subsection{ImageNet}
Figure~\ref{fig:road_imagenet} shows the ROAD curves for the ImageNet dataset. These plots illustrate how model accuracy degrades as top-attributed features are progressively removed, providing a direct evaluation of attribution faithfulness. 

\begin{figure}[!h]
    \centering
    \begin{subfigure}{0.4\linewidth}
        \centering
        \includegraphics[width=\linewidth]{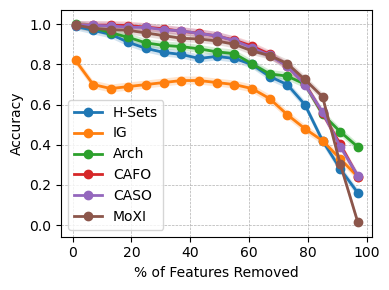}
        \caption{ResNet101}
        \label{fig:resnet101}
    \end{subfigure}
    \hfill
    \begin{subfigure}{0.4\linewidth}
        \centering
        \includegraphics[width=\linewidth]{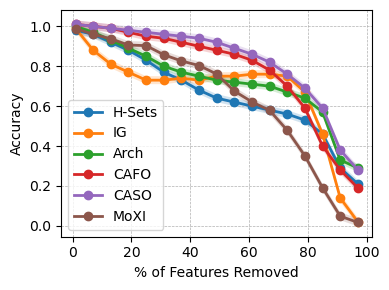}
        \caption{DenseNet121}
        \label{fig:densenet121}
    \end{subfigure}
    \begin{subfigure}{0.4\linewidth}
        \centering
        \includegraphics[width=\linewidth]{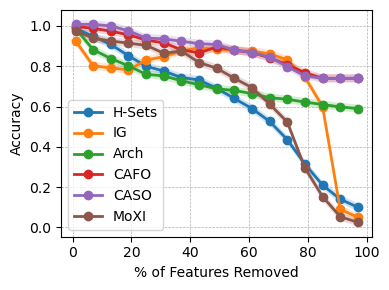}
        \caption{VGG16}
        \label{fig:vgg16}
    \end{subfigure}
    \hfill
    \begin{subfigure}{0.4\linewidth}
        \centering
        \includegraphics[width=\linewidth]{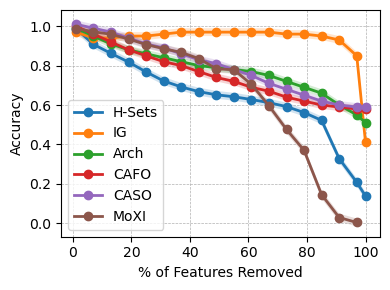}
        \caption{MobileNetV3}
        \label{fig:mobilenet}
    \end{subfigure}
    \caption{ROAD plots: ImageNet}
    \label{fig:road_imagenet}
\end{figure}

\subsection{CUB}
Figure~\ref{fig:roadCUBdataset} presents the \textbf{ROAD} curves for the CUB dataset~\cite{wah2011caltech}, showing how model accuracy degrades as increasingly important features are progressively removed. These plots visually corroborate the $\text{ROAD}_\text{AOPC}$ scores reported in Table~\ref{tab:faithfulnesseval}, where H-Sets achieves the highest scores across all models, reflecting the area over these perturbation curves. 

\begin{figure}[!h]
    \centering
    \begin{subfigure}[t]{0.4\linewidth}
        \centering
        \includegraphics[width=\linewidth]{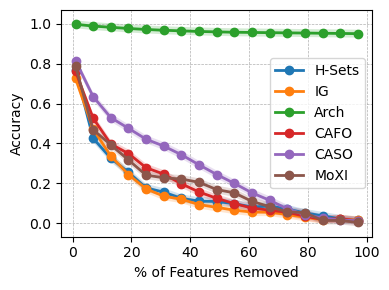}
                \caption{VGG16}
                 \label{fig:vggcub}
    \end{subfigure}
    \hfill
    \begin{subfigure}[t]{0.4\linewidth}
        \centering
        \includegraphics[width=\linewidth]{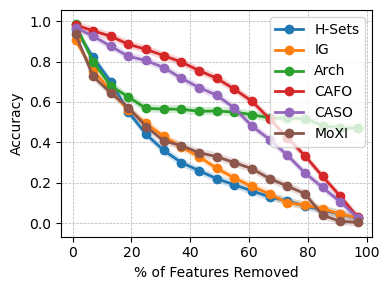}
                \caption{ResNet101}
        \label{fig:resnetcub}
    \end{subfigure}
        \hfill
    \begin{subfigure}[t]{0.4\linewidth}
        \centering
        \includegraphics[width=\linewidth]{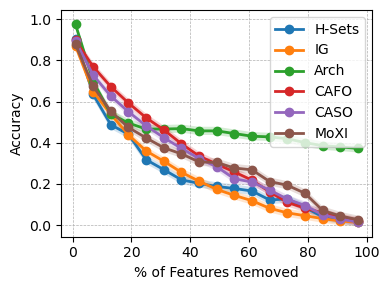}
                \caption{DenseNet121}
        \label{fig:densecub}
    \end{subfigure}
        \hfill
    \begin{subfigure}[t]{0.4\linewidth}
        \centering
        \includegraphics[width=\linewidth]{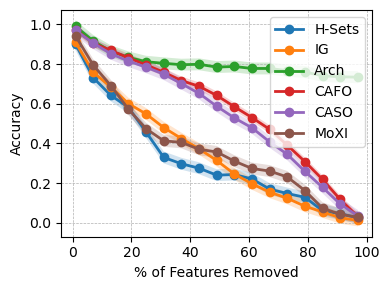}
                \caption{MobileNetV3}
                \label{fig:mobilenetcub}
    \end{subfigure}
        \caption{Faithfulness evaluation using ROAD plots for CUB across different models.}
         \label{fig:roadCUBdataset}
\end{figure}

\section{Ablation: Hyperparameter (CUB)}\label{appendix:ablationHyperparameterCUB}
We conduct an ablation study on the CUB dataset\cite{wah2011caltech} using the ResNet101~\cite{resnet} model to evaluate the robustness of H-Sets to two key hyperparameters: the number of features in the interaction set $\nu$ and the interaction threshold $\mu$ that determines the minimum pairwise Hessian strength.

\begin{table}[ht]
\centering
\caption{Ablation of H-Sets hyperparameters on CUB dataset using MobileNetV3. \textbf{Top}: varying the number of features included in explanations $\nu$. \textbf{Bottom}: varying the Hessian interaction threshold $\mu$.}
\label{tab:ablation-side-by-side}
\begin{minipage}{0.65\linewidth}
\centering
\resizebox{\textwidth}{!}{%
\begin{tabular}{@{}ccc@{}}
\toprule
\textbf{\# Features ($\nu$)} & \textbf{Sparsity} & \textbf{$\text{ROAD}_\text{AOPC}$} \\ \midrule
250 & {0.990} & 0.643 \\
1000 & 0.961 & 0.629 \\
2000 & 0.928 & {0.647} \\
3000 & 0.898 & 0.611 \\
5000 & 0.845 & 0.598 \\
\bottomrule
\end{tabular}%
}
\end{minipage}
\hspace{2em}
\begin{minipage}{0.65\linewidth}
\centering
\resizebox{\textwidth}{!}{%
\begin{tabular}{@{}ccc@{}}
\toprule
\textbf{Threshold ($\mu$)} & \textbf{Sparsity} & \textbf{$\text{ROAD}_\text{AOPC}$} \\ \midrule
0.1 & 0.903 & 0.614 \\
0.2 & 0.903 & 0.596 \\
0.3 & 0.901 & 0.609 \\
0.4 & 0.901 & 0.609 \\
0.5 & 0.898 & 0.610 \\
0.6 & 0.903 & 0.599 \\
0.7 & 0.900 & {0.615} \\
0.8 & {0.971} & 0.607 \\
\bottomrule
\end{tabular}%
}
\end{minipage}
\end{table}

\textbf{Number of features.} Table~\ref{tab:ablation-side-by-side} (top) shows that using very few features (e.g., 250) yields highly sparse maps (high sparsity score), but these can underrepresent the model's behavior, especially on large images like CUB, where objects typically span large spatial regions. As we increase the number of features, sparsity decreases steadily, which is expected as more interacting features are included. Despite this, the measure of faithfulness with $\text{ROAD}_\text{AOPC}$ score remains consistently high across the range, suggesting that our method reliably identifies relevant interactions. However, increasing the number of features raises computational cost, as our method performs Hessian-based interaction detection and set-level directional attributions. We find that selecting 2000–3000 features provides an effective trade-off: the explanations are still sparse and interpretable, ROAD scores remain high, and the computation remains tractable (see Figure~\ref{fig:numfeatures-CUB} for sample explanations).

\begin{figure}[ht]
    \centering
    \includegraphics[width=0.95\linewidth]{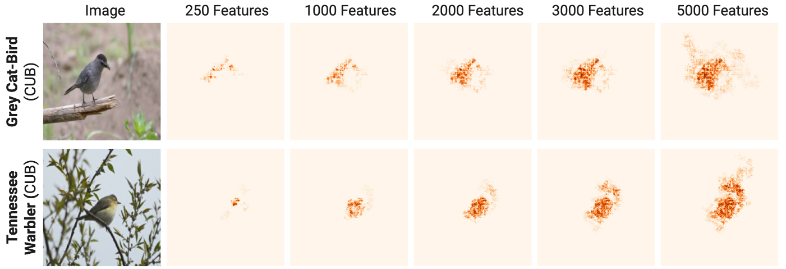}
    \caption{Saliency maps with different values of $\nu$ on CUB.}
    \label{fig:numfeatures-CUB}
\end{figure}

\textbf{Interaction threshold.} Table~\ref{tab:ablation-side-by-side} (bottom) shows the effect of varying the interaction threshold $\mu$, with the number of features fixed at 3000. We observe that both sparsity and $\text{ROAD}_\text{AOPC}$  scores remain stable across a wide range of $\mu$ values, indicating that H-Sets is robust to this hyperparameter. This robustness indicates that H-Sets reliably detects semantically meaningful interactions without being overly sensitive to the exact strength of second-order gradients. From a deployment perspective, this robustness simplifies hyperparameter tuning, making H-Sets practical to apply across datasets and architectures with minimal adjustment.

\end{document}